\documentclass{article}

\usepackage{microtype}
\usepackage{graphicx}
\usepackage{subcaption}
\usepackage{booktabs} 

\PassOptionsToPackage{hyphens}{url}
\usepackage{xurl}
\usepackage{hyperref}

\usepackage[accepted]{sty/icml2026}

\usepackage{amsmath}
\usepackage{amssymb}
\usepackage{mathtools}
\usepackage{amsthm}
\allowdisplaybreaks

\usepackage{siunitx}
\DeclareMathOperator{\HOp}{H}
\DeclareMathOperator{\SOp}{S}
\newcommand{\Hstep}[1]{\HOp\!\left(#1\right)}
\newcommand{\Sign}[1]{\SOp\!\left(#1\right)}
\DeclareMathOperator{\SalphaOp}{S_\alpha}
\DeclareMathOperator{\LinearOp}{Linear}

\DeclareMathOperator{\GLUOp}{GLU}
\DeclareMathOperator{\MLPOp}{MLP}
\DeclareMathOperator{\DropoutOp}{Dropout}
\DeclareMathOperator{\NormOp}{Norm}
\DeclareMathOperator{\SubLayerOp}{SubLayer}
\DeclareMathOperator{\CellOp}{Cell}
\DeclareMathOperator{\CellBOp}{\widetilde{\mathrm{Cell}}}
\newcommand{\Signalpha}[1]{\SalphaOp\!\left(#1\right)}
\newcommand{\Linear}[1]{\LinearOp\!\left(#1\right)}

\newcommand{\GLU}[1]{\GLUOp\!\left(#1\right)}
\newcommand{\MLP}[1]{\MLPOp\!\left(#1\right)}
\newcommand{\Dropout}[1]{\DropoutOp\!\left(#1\right)}
\newcommand{\Norm}[1]{\NormOp\!\left(#1\right)}
\newcommand{\SubLayer}[1]{\SubLayerOp\!\left(#1\right)}
\newcommand{\Cell}[1]{\CellOp\!\left(#1\right)}
\newcommand{\CellB}[1]{\CellBOp\!\left(#1\right)}

\usepackage{multirow}
\usepackage{makecell}

\usepackage[capitalize,noabbrev]{cleveref}

\theoremstyle{plain}
\newtheorem{theorem}{Theorem}[section]
\newtheorem{proposition}[theorem]{Proposition}

\theoremstyle{definition}

\theoremstyle{remark}
\newtheorem{remark}[theorem]{Remark}

\usepackage{xcolor}
\usepackage{colortbl}

\definecolor{best}{RGB}{0,100,0}        
\definecolor{nearbest}{RGB}{34,139,34}  
\definecolor{good}{RGB}{107,142,35}     

\usepackage{placeins,afterpage}

\icmltitlerunning{Improving the Performance and Learning Stability of Parallelizable RNNs Designed for Ultra-Low Power Applications}

\begin{document}

\twocolumn[
    \icmltitle{Improving the Performance and Learning Stability of Parallelizable RNNs Designed for Ultra-Low Power Applications}



  \icmlsetsymbol{equal}{*}

  \begin{icmlauthorlist}
    \icmlauthor{Julien Brandoit}{uliege}
    \icmlauthor{Arthur Fyon}{uliege}
    \icmlauthor{Damien Ernst}{uliege}
    \icmlauthor{Guillaume Drion}{uliege}
  \end{icmlauthorlist}

  \icmlaffiliation{uliege}{Department of Electrical Engineering and Computer Science, University of Liège, Liège, Belgium}

  \icmlcorrespondingauthor{Julien Brandoit}{jbrandoit@uliege.be}
  \icmlcorrespondingauthor{Guillaume Drion}{gdrion@uliege.be}

  \icmlkeywords{Machine Learning, ICML, Recurrent Neural Networks, Bistable Memory Recurrent Unit, Persistent Memory, Parallelizable RNNs, Long-Range Dependencies, Sequence Modeling, Analog Neural Networks}

  \vskip 0.3in
]



\printAffiliationsAndNotice{}  

\begin{abstract}
  Sequence learning is dominated by Transformers and parallelizable recurrent neural networks (RNNs) such as state-space models, yet learning long-term dependencies remains challenging, and state-of-the-art designs trade power consumption for performance. The Bistable Memory Recurrent Unit (BMRU) was introduced to enable hardware-software co-design of ultra-low power RNNs: quantized states with hysteresis provide persistent memory while mapping directly to analog primitives. However, BMRU performance lags behind parallelizable RNNs on complex sequential tasks. In this paper, we identify gradient blocking during state updates as a key limitation and propose a cumulative update formulation that restores gradient flow while preserving persistent memory, creating skip-connections through time. This leads to the Cumulative Memory Recurrent Unit (CMRU) and its relaxed variant, the $\alpha$CMRU. Experiments show that the cumulative formulation dramatically improves convergence stability and reduces initialization sensitivity. The CMRU and $\alpha$CMRU match or outperform Linear Recurrent Units (LRUs) and minimal Gated Recurrent Units (minGRUs) across diverse benchmarks at small model sizes, with particular advantages on tasks requiring discrete long-range retention, while the CMRU retains quantized states, persistent memory, and noise-resilient dynamics essential for analog implementation.
\end{abstract}

\section{Introduction}

The success of deep learning on sequential data has been driven by increasingly large models with increasingly large computational costs. Transformers achieve state-of-the-art performance through attention mechanisms, but their quadratic complexity in sequence length makes them impractical for resource-constrained deployment. Recurrent neural networks (RNNs) offer a more efficient alternative, encoding temporal dependencies in a fixed-size hidden state. However, classical RNNs rely on nonlinear recurrence, creating sequential dependencies that prevent parallel training.

A growing body of work has propelled parallelizable RNNs by adopting linear state dynamics as a deliberate architectural constraint. This design choice precludes nonlinear recurrence but is the structural property these models leverage to parallelize computation via associative scan algorithms. State-space models (SSMs) and parallelizable gated RNNs exploit this linearity to train increasingly large architectures, matching or exceeding Transformer performance on many benchmarks.

Beyond this notable boost in performance, modern parallelizable RNNs face two fundamental challenges. The first concerns hardware efficiency. Current RNN implementations, whether linear or nonlinear, are emulated on digital architectures: CPUs, GPUs, microcontrollers, or FPGAs~\cite{fedorov2019sparse, chang2015recurrent}. Even optimized TinyML deployments consume milliwatts of power~\cite{banbury2021mlperf, warden2019tinyml}. Achieving true sub-microwatt operation, as required for \textit{always-on} sensors or biomedical implants with decade-long battery life, demands RNNs with computational primitives that map directly onto efficient analog circuits rather than being emulated in digital logic. The second concerns memory. The contractive dynamics adopted by parallelizable RNNs, with state eigenvalues kept inside the unit circle to ensure stability, force states to decay exponentially toward a unique equilibrium, a property known as \emph{fading memory}. While fading memory suffices for many tasks, it precludes \emph{persistent memory}, which is the ability to retain information indefinitely until explicitly overwritten. Persistent memory can only be emulated by such architectures through increased network size, which is to be avoided in low-power applications. \emph{True, efficient persistent memory requires multistability, which contractive dynamics cannot exhibit.}

Bistable Memory Recurrent Units (BMRUs) have been specifically designed to address both challenges. BMRUs enable persistent memory through bistability between two \emph{quantized} hidden states, while remaining parallelizable via associative scans. This bistable, quantized structure maps directly to Schmitt triggers, a well-known analog primitive. It enables ultra-low power implementation in monolithic CMOS technology: a recent design achieves $>$90\% accuracy at \SI{100}{nW} for keyword spotting, three orders of magnitude below digital alternatives~\cite{fyon2026fullytunableultralowpower, Fyon2026AnalogRecurrent}. Quantized states also block analog noise propagation, preventing error accumulation across timesteps. However, persistent memory comes at the cost of expressivity. BMRUs lack the fading memory that SSMs and gated RNNs exploit for complex sequential computations, limiting performance on tasks requiring more than long-term retention. This raises a natural question: \emph{can BMRUs maintain their persistent memory and hardware co-design properties while approaching the performance of SSMs and gated parallelizable RNNs on more demanding sequential tasks?}

In this paper, we tackle this question by first identifying a key limitation of BMRUs: the full state replacement in update mode blocks gradient flow during training. We introduce a cumulative update formulation that restores gradient flow, dramatically improving convergence and reducing initialization sensitivity while preserving persistent memory. Building on this formulation, we propose the Cumulative Memory Recurrent Unit (CMRU) and the $\alpha$CMRU. The CMRU retains quantized states and maps directly to analog primitives, now corresponding to a hysteretic charge accumulator, making it suitable for ultra-low power analog implementations. The $\alpha$CMRU relaxes quantization constraints, providing increased expressivity for large-scale digital machine learning. We evaluate our approach with a focus on small model sizes and interpretability, motivated by the target application of ultra-low power analog electronics. Our experiments demonstrate that the CMRU and $\alpha$CMRU bridge the performance gap with SSMs and parallelizable gated RNNs while addressing distinct deployment scenarios: analog hardware for the CMRU and classical digital systems for the $\alpha$CMRU.

Our contributions can be summarized as follows:
\begin{enumerate}
    \item We identify gradient blocking during BMRU state updates as the key training bottleneck and show that the cumulative update formulation resolves it, making parallelizable persistent memory a reliable and competitive machine
    learning primitive for the first time.

    \item We introduce the Cumulative Memory Recurrent Unit (CMRU) and its relaxed
    variant, the $\alpha$CMRU, as novel parallelizable RNN cells providing persistent
    memory. The CMRU is structurally compatible with associative-scan pipelines
    and can serve as a drop-in persistent-memory block in Mamba-style architectures. The ($\alpha$)CMRU admits a selective state-space representation in which the
    hyperparameter $\varepsilon$ directly controls the eigenvalues of the state
    transition matrix during updates, unifying the cumulative and bistable regimes
    under a single parameterization and connecting persistent memory to the broader
    theory of linear recurrent models.

    \item Through systematic benchmarking, we establish a complementary performance
    pattern: CMRU variants outperform fading-memory cells on tasks requiring
    discrete state preservation, while fading-memory cells retain an advantage on
    tasks requiring smooth temporal integration. This specialization motivates hybrid
    architectures combining both memory types.

    \item We demonstrate practical ultra-low power deployment: a single-layer CMRU
    with hidden state dimension $d \leq 16$ achieves $>$95\% accuracy on keyword spotting benchmarks.
    The $\varepsilon$-annealing training strategy produces models that deploy
    directly on existing validated ultra-low power hardware~\cite{Fyon2026AnalogRecurrent}
    at sub-microwatt power levels. We additionally provide a proof-of-concept
    native CMRU analog circuit, validated in subthreshold CMOS simulation.
\end{enumerate}

\section{Background}

RNNs encode time-dependencies in a recurrent hidden state $h_t$ whose update depends on the previous hidden state $h_{t-1}$ and current observation of a time-series $x_t$. It writes
\begin{equation}
    h_t = f_\theta(x_t, h_{t-1}; \theta),\,\forall t \geq 1\,,
\end{equation}
where $f_\theta$ represents the state transition equation of the RNN and $\theta$ the learnable parameters of the network. This recurrent connection enables information persistence (or \emph{memory}) across time through encoding in its hidden state. Nonlinear versions such as LSTM~\cite{hochreiter1997longshorttermmemory} and GRUs~\cite{cho2014learningphraserepresentationsusing} have dominated sequence modeling for decades. However, their sequential dependencies, arising from nonlinear gating mechanisms, prevent parallel training and energy-efficient analog implementation \cite{martin2018parallelizinglinearrecurrentneural, heydari2025tiny}.

\subsection{Related Work on Parallelizable RNNs}
Recent efforts have led to the emergence of novel RNN architectures that are parallelizable over the sequence length with competitive performance on sequential benchmarks. 

State-space models (SSMs) and linear RNNs exploit linear state dynamics (the $f_\theta$ is affine in $h_{t-1}$) to achieve parallelizable training, either through convolutional representations \cite{gu2021combiningrecurrentconvolutionalcontinuoustime, gu2022efficientlymodelinglongsequences} or the parallel scan algorithm \cite{blelloch1989scans, martin2018parallelizinglinearrecurrentneural, smith2022simplifiedstatespacemodels, gu2024mambalineartimesequencemodeling}. The Linear Recurrent Unit (LRU) \cite{orvieto2023resurrectingrecurrentneuralnetworks} distills this family to its essential form: a diagonal linear recurrence with complex-valued states and exponential parameterization for stability. We adopt LRU as our representative of this family of architectures in our experiments for its simplicity and competitive performance. This cell is presented in detail in \cref{subsection:lru}.

Gated RNNs can also achieve parallelizable training by removing the state-dependency in their gating mechanisms, i.e. computing gate values from the input alone, $z_t = \sigma(W_z x_t)$. It yields an update of the form $h_t = (1 - z_t) \odot h_{t-1} + z_t \odot \tilde{h}_t,$
where $\tilde{h}_t$ depends only on $x_t$. As in SSMs, this is a linear recurrence in $h_t$ (the $f_\theta$ is affine in $h_{t-1}$) and can thus be parallelized via the scan algorithm. Several such architectures have been proposed, including xLSTM \cite{beck2024xlstmextendedlongshortterm}, HGRN \cite{qin2023hierarchicallygatedrecurrentneural}, MatMul-free LM \cite{zhu2025scalablematmulfreelanguagemodeling} and GILR \cite{martin2018parallelizinglinearrecurrentneural}, though some only achieve partial parallelization. We adopt minGRU \cite{feng2024weregrudminimalgatedrecurrent} as our representative of this family of architectures in our experiments: it strips the GRU to its minimal form (a single input-dependent gate and no reset gate) making it one of the simplest fully parallelizable gated architectures. This cell is presented in detail in \cref{subsection:mingru}.

In all these models, their structure implies monostability, representative of \textit{fading memory} \cite{boyd1985fadingmemory}, as they force the eigenvalues of the discrete update laws to strictly lie within the unit circle. A broader discussion of related work, including hardware-oriented and hybrid architectures, is provided in \cref{app:related_work}. 

\subsection{The Bistable Memory Recurrent Unit}

The Bistable Memory Recurrent Unit (BMRU) has been recently introduced as an alternative RNN cell that achieves persistent memory while preserving parallelizability~\cite{degeeter2026parallelizablememoryrecurrentunits}. The BMRU hidden state is \emph{quantized}: each component of the hidden state vector $h_t$ can only take one of two values, $+\alpha$ or $-\alpha$. This binary quantization is the signature of bistability. The BMRU dynamics are given by:
\begin{align}
    \hat{h}_t &= W_x x_t + b_x\,, \label{eq:bmru_candidate}\\
    \beta_t &= \left|W_\beta x_t + b_\beta\right|\,, \label{eq:bmru_threshold}\\
    z_t &= \Hstep{|\hat{h}_t| - \beta_t}\,, \label{eq:bmru_gate}\\
    h_t &= z_t \odot \Sign{\hat{h}_t} \odot \alpha + (1 - z_t) \odot h_{t-1}\,, \label{eq:bmru_update}
\end{align}
where $\Hstep{\cdot}$ denotes the Heaviside step function, $\Sign{\cdot}$ is the sign function, $W_x$ and $W_\beta$ are learnable weight matrices, and $b_x$, $b_\beta$, and $\alpha$ are learnable parameter vectors. This formulation implements adaptive thresholding: when the candidate magnitude $|\hat{h}_t|$ exceeds the input-dependent threshold $\beta_t$ 
, the hidden state switches to $\alpha\Sign{\hat{h}_t}$; otherwise, it remains unchanged. The quantized, bistable nature of the state enables persistent memory: information is retained indefinitely until an input actively triggers a switch.

The interest of BMRUs lies in the direct mapping onto the Schmitt trigger, an analog primitive. This enables ultra-low power analog implementations in monolithic CMOS technology~\cite{Fyon2026AnalogRecurrent,fyon2026fullytunableultralowpower}. In contrast, fading-memory architectures such as LRUs are typically emulated in digital hardware for low-power applications (microcontrollers, FPGAs) \cite{banbury2021mlperf, warden2019tinyml, fedorov2019sparse, chang2015recurrent}; a true analog implementation would require capacitors to realize the exponential decay, precluding monolithic integration. 

\section{Extending the BMRU: Continuous States and Cumulative Updates}
The original BMRU achieves persistent memory through quantized bistable states, but its design introduces two potential limitations: (i) the restriction to binary values $\pm\alpha$ may limit representational capacity, and (ii) the state reset during updates blocks gradient flow, hindering learning. In this section, we introduce the Cumulative Memory Recurrent Unit (CMRU) and its variant the $\alpha$CMRU. The CMRU addresses the gradient flow limitation through a cumulative update formulation while retaining quantized states for analog implementation. The $\alpha$CMRU additionally relaxes quantization constraints for large-scale digital machine learning. 

\subsection{Restoring Gradient Flow: Cumulative Updates}\label{subsection:vanishing}
The BMRU suffers from a gradient blocking problem during state updates. The cell operates in two modes depending on the gate $z_t$. In \emph{retain mode} ($z_t = 0$, when $|\hat{h}_t| < \beta_t$), the hidden state is copied forward unchanged, enabling perfect gradient propagation (see \cref{eq:bmru_update}).
However, in \emph{update mode} ($z_t = 1$, when $|\hat{h}_t| \geq \beta_t$), the previous state is erased and overwritten, completely blocking gradient flow:
\begin{equation}
    \left.\frac{\partial h_t}{\partial h_{t-1}}\right|_{z_t=1} = 0\,.
\end{equation}
This prevents error signals from reaching earlier timesteps through update operations, making it difficult for the model to learn effective update strategies via backpropagation through time. The problem is most severe during early training, before the gating mechanism has learned when to retain versus update, resulting in poor convergence and high initialization sensitivity.

\textbf{Cumulative update formulation.}
We address gradient blocking by enabling partial state preservation during updates. Introducing a hyperparameter $\varepsilon \in [0,1]$ that controls the fraction of previous state retained, we define the Cumulative Memory Recurrent Unit (CMRU) through the modified update equation:
\begin{equation}\label{eq:cmru}
    h_t =  z_t \odot \left[\Sign{\hat{h}_t} \odot \alpha + \varepsilon h_{t-1}\right] + (1-z_t) \odot h_{t-1}\,.
\end{equation}

The term $\varepsilon h_{t-1}$ is active only during updates ($z_t = 1$), creating a temporal skip connection that restores gradient flow (from \cref{eq:cmru}):
\begin{equation}
    \left.\frac{\partial h_t}{\partial h_{t-1}}\right|_{z_t=1} = \varepsilon\,.
\end{equation}

The hyperparameter $\varepsilon$ interpolates between three regimes.

When $\varepsilon = 0$, the formulation reduces to the original BMRU with full state replacement.

When $0 < \varepsilon < 1$, the model exhibits \emph{non-instantaneous fading memory}. Unlike LRU or minGRU, where the state decays at every timestep regardless of input, here fading occurs only during update events. Consider a sequence of $k$ successive updates at times $t_1, t_2, \ldots, t_k$: the contribution of the state at $t_1$ to $h_{t_k}$ is scaled by $\varepsilon^{k-1}$. Memory thus fades exponentially with the \emph{number of updates}, not with elapsed time. Between updates, the state is preserved perfectly. This creates an event-driven forgetting mechanism: salient inputs (those triggering updates) gradually overwrite older information, while periods of low activity preserve memory indefinitely.

When $\varepsilon = 1$, the cell becomes a \emph{selective integrator}:
\begin{equation}\label{eq:integrator}
    \left. h_t \right|_{\varepsilon=1} = h_{t-1} + z_t \odot \Sign{\hat{h}_t} \odot \alpha\,,
\end{equation}
accumulating quantized inputs into the hidden state with perfect gradient flow ($\tfrac{\partial h_t}{\partial h_{t-1}} = 1$).

Critically, the persistent memory property in retain mode ($z_t = 0$) is unchanged for all values of $\varepsilon$. 

\subsection{Relaxing Quantization: the $\boldsymbol{\alpha}$CMRU Variant}

To investigate whether the CMRU fixed quantization scale limits representational capacity, we introduce the $\alpha$CMRU, which replaces the fixed scale $\alpha$ with an input-dependent function:
\begin{equation}\label{eq:alpha_cmru}
    \alpha_t = W_\alpha x_t + b_\alpha\,,
\end{equation}
where $W_\alpha$ and $b_\alpha$ are learnable parameters. The state update equation becomes:
\begin{equation}
    h_t = z_t \odot \left[\Sign{\hat{h}_t} \odot \alpha_t + \varepsilon h_{t-1}\right] + (1 - z_t) \odot h_{t-1}\,.
\end{equation}
Unlike the CMRU, whose hidden states are quantized, the $\alpha$CMRU can represent arbitrary values in $\mathbb{R}^d$.

The two variants address different deployment scenarios with complementary trade-offs. For ultra-low power analog implementations, the CMRU quantized states offer inherent noise resilience. 
For classical machine learning on digital hardware, the $\alpha$CMRU continuous representation removes any potential precision constraints, providing increased expressivity for large-scale applications.

\subsection{Fixed Point Structure and Representational Capacity}

The cumulative formulation fundamentally changes which states the cell can represent. 

In the original BMRU ($\varepsilon = 0$), each update overwrites the state with $\pm\alpha$. The cell can therefore only occupy $2^d$ distinct states, where $d$ is the hidden dimension (one binary choice per dimension). This is the quantized bistable structure.

With $\varepsilon = 1$, updates no longer overwrite but \emph{accumulate}: each update adds $\pm\alpha$ to the current state. The reachable states thus form a regular lattice at integer multiples of $\alpha$, as illustrated in \cref{fig:fixed_points}. This expands representational capacity from $2^d$ discrete states to countably infinite attractors, while maintaining a fixed precision: successive states differ by at least $\|\alpha\|$.

For the $\alpha$CMRU with $\varepsilon = 1$, each update accumulates a different input-dependent value $\alpha_t$. Since the $\alpha_t$ vary continuously with the input, the set of reachable states becomes dense in $\mathbb{R}^d$: the cell can in principle approximate any target state with arbitrary precision.

\begin{figure}[thbp!]
\centering
    \vskip 0in
    \centerline{\includegraphics[width=\columnwidth]{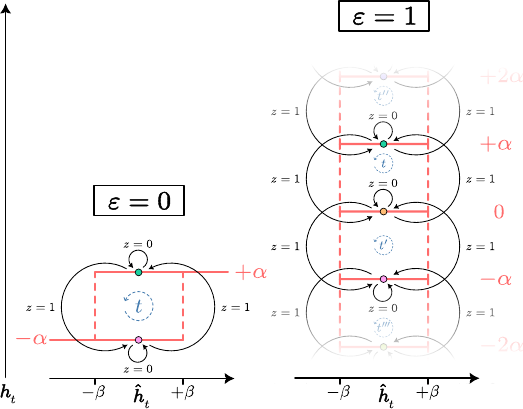}}
    \caption{
  Fixed point structure of hidden state $h_t$ for CMRU as a function of $\varepsilon$.
  \textbf{Left} ($\varepsilon=0$): Original BMRU with two discrete stable fixed points at $h^* = \pm\alpha$.
  \textbf{Right} ($\varepsilon=1$): CMRU cumulative formulation expands reachable fixed points to all integer multiples $h^* \in \{\pm k\alpha : k \in \mathbb{N}\}$, creating countably infinite stable attractors. Dashed vertical lines at $\pm\beta$ indicate candidate state thresholds for gating.
}
    \label{fig:fixed_points}
\end{figure}

\subsection{Experimental Framework}\label{section:protocol}
Comparing RNN architectures solely on peak benchmark accuracy can be misleading: an architecture that achieves state-of-the-art results under careful tuning may prove brittle in practice if its performance is sensitive to weight initialization or requires extensive hyperparameter search for each new task \cite{henderson2019deepreinforcementlearningmatters, bouthillier2021accountingvariancemachinelearning}. To provide a more complete picture, we characterize each architecture through three complementary properties:

\textbf{(i) Peak performance}: the best accuracy achievable by the architecture, reflecting its expressive capacity;

\textbf{(ii) Initialization sensitivity}: the variability in performance across different random seeds, indicating training reliability;

\textbf{(iii) Hyperparameter robustness}: the ability to perform well across diverse tasks without task-specific tuning, facilitating practical and accessible use across different machine learning domains.

We detail our experimental strategy below. 

\textbf{Common backbone architecture.}
To isolate the effect of the recurrent cell, all cells (LRU, minGRU, CMRU, and $\alpha$CMRU) are integrated into an identical backbone architecture. The backbone interleaves recurrent layers with point-wise multi-layer perceptrons (MLPs), skip connections, and normalization layers, as illustrated in \cref{fig:backbone}. Full architecture and training details are provided in \cref{app:training}.

\textbf{No hyperparameter tuning.}
Critically, we perform no hyperparameter tuning across tasks. A single set of hyperparameters is used for all experiments. The validation set serves exclusively for early stopping and checkpoint selection; it is never used to inform hyperparameter decisions. This protocol ensures that our results reflect the inherent robustness of each architecture rather than the outcome of task-specific optimization, and allows us to assess property~(iii) implicitly through consistent performance across diverse tasks.

\textbf{Evaluation protocol.}
Properties (i) and (ii) are assessed by training each configuration with five independent random seeds and evaluating the best checkpoint on the test set. We report accuracy for classification benchmarks and mean absolute error (MAE) for regression benchmarks. Results are summarized as mean $\pm$ min--max intervals, capturing both peak performance and initialization sensitivity. Numerical values are tabulated in \cref{appendix:tabulated_result}, and raw training logs are publicly available alongside the code.\footnote{\url{https://github.com/julienbrandoit/ICML2026---Improving-Performance-and-Stability-of-Ultra-Low-Power-RNNs.git}}

\section{Experimental Investigation of the ($\alpha$)CMRU}

\subsection{Effect of $\boldsymbol{\varepsilon}$ on Gradient Flow and Convergence}
Throughout this section, we evaluate models using \textit{last pooling}. 
This choice tests both retention of information over the entire sequence and the ability to learn from long-range dependencies. 
Detailed descriptions of all benchmark tasks used in this work are provided in  \cref{app:tasks}.

We first investigate the impact of varying $\varepsilon$ on the sequential MNIST (sMNIST) classification task, where models classify handwritten digits presented pixel-by-pixel in raster scan order. 
Increasing $\varepsilon$ from 0 to 1 yields consistent improvements in mean accuracy and initialization stability across both CMRU and $\alpha$CMRU variants (\cref{figure:accuracy_smnist_combined}, left panel). For CMRU with $d=32$, accuracy improves from approximately 30\% at $\varepsilon=0$ (i.e., the original BMRU) to 96\% at $\varepsilon=1$, while cross-seed variability decreases dramatically. The $\alpha$CMRU exhibits similar trends, achieving 97\% accuracy at $\varepsilon=1$ with minimal variance. These improvements occur despite different quantization mechanisms (fixed versus input-dependent), confirming that gradient flow through cumulative updates (not representational capacity) is the critical factor enabling reliable convergence. Peak performance occurs at $\varepsilon=1$, where gradient magnitude equals unity in both retain and update modes. This observation remains consistent across different permutations of the sequential MNIST task (see \cref{figure:pmnist_challenging} in \cref{subsection:pmnist_results}).

\begin{figure*}[htbp!]
\centering
  \vskip 0in
    \includegraphics[width=\textwidth]{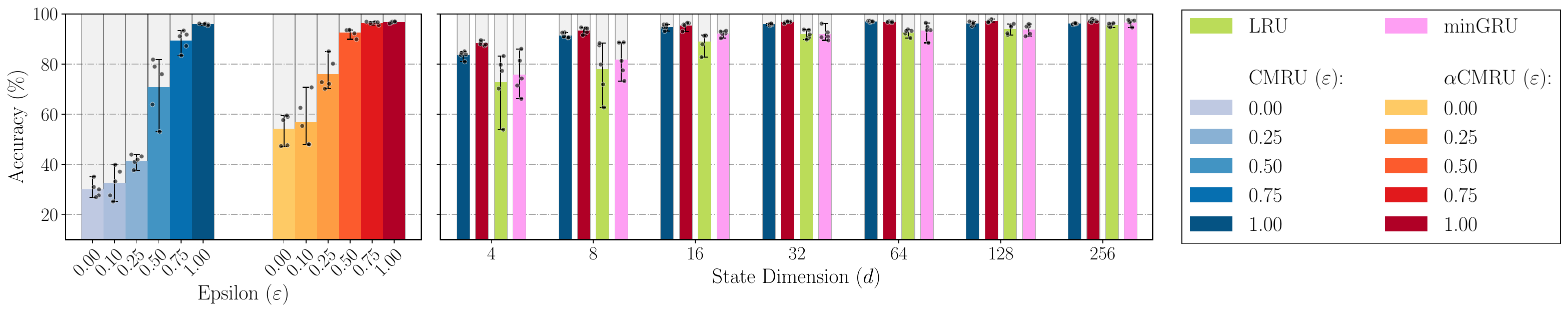}
    \caption{
  Performance on sequential MNIST (sMNIST) classification task.
  Results show mean accuracy across five random initializations with min--max error bars.
  \textbf{Left:} Effect of hyperparameter $\varepsilon \in \{0.00, 0.25, 0.50, 0.75, 1.00\}$ on CMRU and $\alpha$CMRU ($d=32$, single layer, last pooling).
  CMRU with $\varepsilon=0$ corresponds to BMRU.
  Increasing $\varepsilon$ consistently improves mean accuracy and reduces variability.
  \textbf{Right:} Scaling across state dimensions $d \in \{4, 8, 16, 32, 64, 128, 256\}$ for CMRU ($\varepsilon=1$), $\alpha$CMRU ($\varepsilon=1$), LRU, and minGRU (single layer).
  CMRU and $\alpha$CMRU show more robust scaling with higher accuracy and lower variability, especially at smaller dimensions.
}
    \label{figure:accuracy_smnist_combined}
\end{figure*}

\subsection{Scaling Behavior and Initialization Robustness}

Beyond addressing gradient flow, we investigate whether cumulative updates improve robustness across different model capacities. We compare performance on sMNIST across different state dimensions $d$ (\cref{figure:accuracy_smnist_combined}, right panel).

Both CMRU and $\alpha$CMRU with $\varepsilon=1$ achieve higher accuracy and lower cross-seed variability than LRU and minGRU across all tested dimensions. The advantage is most pronounced at small state sizes, which is critical for ultra-low power implementations: at $d=4$, CMRU achieves approximately 84\% accuracy with minimal variance, while LRU and minGRU exhibit substantially higher variability and lower mean performance. At $d=256$, all architectures converge to similar accuracy (96--97\%), but the CMRU variants maintain lower variance.

This reduced initialization sensitivity across all scales confirms that cumulative updates yield more stable optimization landscapes, an essential property for hardware-constrained deployments and large-scale machine learning where training costs are prohibitive. The consistent behavior across dimensions suggests that these improvements are fundamental to the architecture rather than artifacts of a particular capacity regime.

\subsection{Persistent Memory under Continuous Noise}\label{subsection:copyfirstnoisy}

We next evaluate persistent memory and state efficiency directly using a copy-first-input task with continuous noisy inputs. A scalar value $x_0 \sim \mathcal{U}(-1, 1)$ is presented at $t=0$ and must be reproduced at the final timestep after processing uniformly random noise $x_t \sim \mathcal{U}(-1, 1)$ for all $t > 0$. This design tests both long-term retention and the ability to reject irrelevant information. For the quantized CMRU, this task also reveals how efficiently the architecture exploits its discrete state space, by comparing its error against the theoretical quantization limit $\mathcal{E}^*_{\text{MAE}}$ (
see \cref{app:quantized}).

\begin{figure}[htbp!]
  \vskip 0in
\centering    \includegraphics[width=\columnwidth]{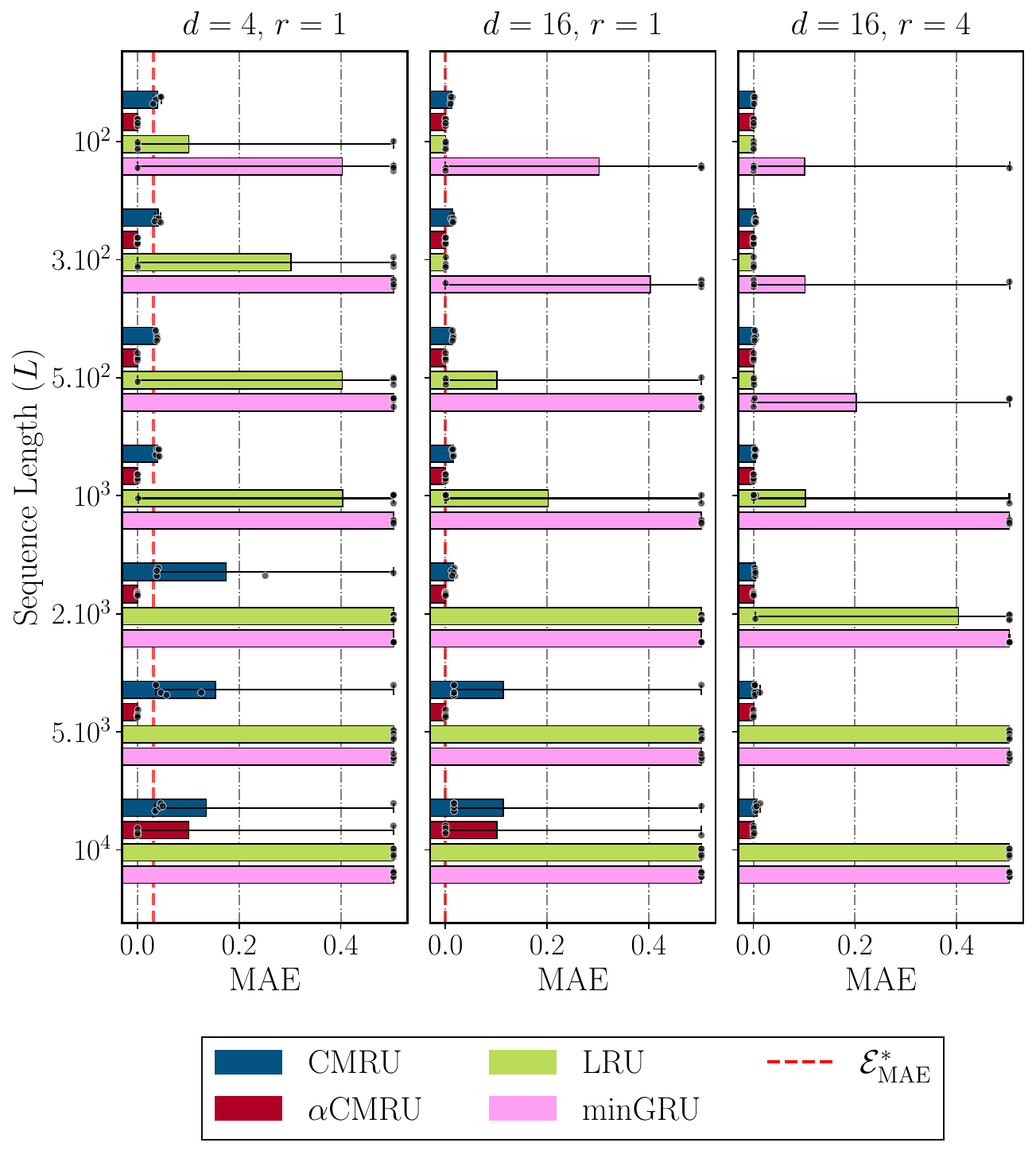}
\caption{
  MAE on copy-first-input (continuous, noisy) versus sequence length $L$.
  \textbf{Left:} $r=1$, $d=4$.
  \textbf{Middle:} $r=1$, $d=16$.
  \textbf{Right:} $r=4$, $d=16$.
  Dashed red lines: CMRU quantization limits $\mathcal{E}^*_{\text{MAE}}$.
  Results shown as mean with min--max error bars over five seeds.
}
    \label{figure:mae_copy_first_combined}
\end{figure}

CMRU with $\varepsilon=1$ maintains consistent performance across sequence lengths $L$ 
and capacity configurations (\cref{figure:mae_copy_first_combined}). In the most constrained setting ($d=4$, single layer and long sequence lengths), CMRU fails to solve the task robustly. This is due to limited noise rejection rather than gradient propagation issues, as these failures disappear in a noise-free version of the task (\cref{figure:copy_first_combined}, right). CMRU achieves MAE near the theoretical limit $\mathcal{E}^*_{\text{MAE}}$ at all sequence lengths, confirming that cumulative updates enable near-optimal memory encoding despite the discrete representation. We further validated this property on a 15-classes first-input classification task with $d=4$: CMRU with $\varepsilon=1$ achieves 100\% accuracy regardless of sequence length (\cref{figure:copy_first_combined}, left), showing optimal use of the hidden state.

The $\alpha$CMRU with $\varepsilon=1$ removes the quantization constraint entirely, achieving near-perfect performance across all configurations. This confirms that input-dependent scaling offers a path to arbitrary precision when fixed quantization is insufficient.

In contrast, LRU and minGRU degrade substantially as sequence length increases, with high cross-seed variability, and in the presence or absence of noise. This is not merely an optimization issue: their dynamics fundamentally preclude perfect retention. LRU applies exponential decay through eigenvalues with magnitude less than one, causing stored information to fade continuously. minGRU lacks a true retention mode: even when the gate is near zero, continuous input processing prevents infinite memory preservation. At $L=10000$, both architectures fail entirely. Even scaling to $r=4$ layers with $d=16$ does not consistently match a single-layer CMRU, confirming that the limitation is architectural, not capacity-related.

We further validate this advantage at larger scale. With $r=6$ layers and $d=256$, CMRU and $\alpha$CMRU solve the task (continuous, noisy, $L=5000$) robustly across all five seeds, while LRU and minGRU fail entirely. This confirms that the persistent memory advantage is not a small-model artifact (see \cref{app:copyfirst_large_scale}).


\textbf{Summary and Choice of $\boldsymbol{\varepsilon}$ for Subsequent Experiments}

Across all tasks and model capacities, increasing $\varepsilon$ consistently improves convergence stability, reduces initialization sensitivity, and enhances final performance. The best results occur at $\varepsilon=1$, where gradient magnitude equals unity in both retain and update modes. Based on these findings, \textit{we fix $\mathit{\varepsilon=1}$ for all subsequent experiments.}

\section{Low-Capacity Scaling on Sequential Benchmarks}

The preceding experiments establish that cumulative updates enable robust persistent memory. We now assess whether CMRU variants maintain competitive performance on tasks demanding more than long-term retention, and whether this competitiveness extends to practically relevant applications at minimal capacity.

\subsection{Classical Sequential Benchmarks}

We evaluate four benchmarks from the Long Range Arena suite~\cite{tay2020longrangearenabenchmark}: IMDb sentiment classification, sequential CIFAR-10, ListOps, and Pathfinder. These benchmarks span a range of computational requirements, from smooth temporal integration to compositional discrete reasoning and long-range spatial path-tracing. We employ mean pooling, following standard practice. We evaluate architectures with fixed state dimension $d=32$ and varying depth $r \in \{1, 3, 6\}$ recurrent layers, except for Pathfinder where we use $r=3$ only. Results are presented in \cref{tab:part3_tasks_combined}.

\begin{table}[t!]
\caption{
  Accuracy (\%) on classical sequential benchmarks (IMDb, sCIFAR10, ListOps, Pathfinder).
  Results for different $(r, d)$ shown as mean with min--max range.
  Pathfinder uses $(r,d)=(3,32)$ only; chance level is 50\%.
  \textcolor{best}{\textbf{Dark green}}: best performance; \textcolor{nearbest}{medium green}: within 1\%; \textcolor{good}{olive green}: within 2\%.
}
\label{tab:part3_tasks_combined}
\centering
    \begin{small}
      \begin{sc}
        \setlength{\tabcolsep}{3pt}
        \begin{tabular}{lcccc}
          \toprule
          $(r, d)$ & $\alpha$CMRU & CMRU & LRU & minGRU \\
          \midrule
          \multicolumn{5}{c}{\textbf{IMDb}} \\
          \midrule
          (1, 32) & \textcolor{good}{\shortstack{64.90 \\ {\scriptsize [61.00; 66.31]}}} & \textcolor{good}{\shortstack{64.84 \\ {\scriptsize [64.19; 65.44]}}} & \textbf{\textcolor{best}{\shortstack{66.54 \\ {\scriptsize [64.16; 70.31]}}}} & \textcolor{nearbest}{\shortstack{65.88 \\ {\scriptsize [64.31; 67.28]}}} \\
          (3, 32) & \textcolor{good}{\shortstack{64.62 \\ {\scriptsize [63.41; 65.56]}}} & \textcolor{nearbest}{\shortstack{65.18 \\ {\scriptsize [64.34; 66.19]}}} & \textbf{\textcolor{best}{\shortstack{66.84 \\ {\scriptsize [64.16; 70.25]}}}} & \textcolor{nearbest}{\shortstack{66.53 \\ {\scriptsize [64.38; 70.12]}}} \\
          (6, 32) & \textcolor{good}{\shortstack{64.13 \\ {\scriptsize [60.19; 66.31]}}} & \textcolor{nearbest}{\shortstack{65.04 \\ {\scriptsize [63.28; 66.75]}}} & \textbf{\textcolor{best}{\shortstack{67.03 \\ {\scriptsize [61.47; 69.16]}}}} & \textcolor{nearbest}{\shortstack{66.21 \\ {\scriptsize [63.81; 67.69]}}} \\
          \midrule
          \multicolumn{5}{c}{\textbf{sCIFAR10}} \\
          \midrule
          (1, 32) & \shortstack{52.12 \\ {\scriptsize [51.34; 53.44]}} & \textcolor{good}{\shortstack{56.27 \\ {\scriptsize [55.53; 57.31]}}} & \textbf{\textcolor{best}{\shortstack{64.49 \\ {\scriptsize [63.06; 66.72]}}}} & \textcolor{nearbest}{\shortstack{58.64 \\ {\scriptsize [57.53; 59.31]}}} \\
          (3, 32) & \textcolor{good}{\shortstack{58.64 \\ {\scriptsize [57.69; 60.09]}}} & \textcolor{nearbest}{\shortstack{59.82 \\ {\scriptsize [55.25; 61.66]}}} & \textcolor{nearbest}{\shortstack{59.99 \\ {\scriptsize [58.94; 61.44]}}} & \textbf{\textcolor{best}{\shortstack{60.95 \\ {\scriptsize [60.44; 61.88]}}}} \\
          (6, 32) & \textcolor{good}{\shortstack{59.71 \\ {\scriptsize [57.69; 61.38]}}} & \shortstack{58.81 \\ {\scriptsize [57.91; 59.59]}} & \textbf{\textcolor{best}{\shortstack{62.34 \\ {\scriptsize [58.75; 65.19]}}}} & \textcolor{nearbest}{\shortstack{62.12 \\ {\scriptsize [61.06; 62.84]}}} \\
          \midrule
          \multicolumn{5}{c}{\textbf{ListOps}} \\
          \midrule
          (1, 32) & \textcolor{nearbest}{\shortstack{38.74 \\ {\scriptsize [36.91; 39.75]}}} & \textcolor{nearbest}{\shortstack{38.84 \\ {\scriptsize [36.84; 39.75]}}} & \textcolor{nearbest}{\shortstack{38.64 \\ {\scriptsize [37.03; 39.75]}}} & \textbf{\textcolor{best}{\shortstack{39.40 \\ {\scriptsize [37.69; 40.94]}}}} \\
          (3, 32) & \textbf{\textcolor{best}{\shortstack{60.86 \\ {\scriptsize [59.94; 62.09]}}}} & \textcolor{nearbest}{\shortstack{60.42 \\ {\scriptsize [59.44; 61.03]}}} & \shortstack{40.45 \\ {\scriptsize [37.75; 48.00]}} & \textcolor{good}{\shortstack{55.07 \\ {\scriptsize [39.06; 60.28]}}} \\
          (6, 32) & \textcolor{nearbest}{\shortstack{61.36 \\ {\scriptsize [59.28; 62.94]}}} & \textbf{\textcolor{best}{\shortstack{61.39 \\ {\scriptsize [60.38; 62.16]}}}} & \shortstack{39.37 \\ {\scriptsize [38.94; 39.91]}} & \textcolor{good}{\shortstack{59.69 \\ {\scriptsize [57.59; 61.03]}}} \\
          \midrule
          \multicolumn{5}{c}{\textbf{Pathfinder}} \\
          \midrule
          (3, 32) & \textbf{\textcolor{best}{\shortstack{92.00 \\ {\scriptsize [89.80; 93.70]}}}} & \textcolor{nearbest}{\shortstack{90.20 \\ {\scriptsize [87.30; 92.10]}}} & \shortstack{50.14 \\ {\scriptsize [49.00; 51.50]}} & \shortstack{49.30 \\ {\scriptsize [48.10; 51.20]}} \\
          \bottomrule
        \end{tabular}
      \end{sc}
    \end{small}
  \vskip -0.1in
\end{table}

\begin{table*}[htbp!]
\caption{
  Cross-entropy loss (lower is better) on Shakespeare character-level
  language modeling.
}
\label{tab:lm_shakespeare}
\centering
\begin{small}
  \begin{sc}
    \setlength{\tabcolsep}{4pt}
    \begin{tabular}{lcccc|cc}
      \toprule
      & $\alpha$CMRU & CMRU & LRU & minGRU & \shortstack{minGRU\\+CMRU} & \shortstack{minGRU\\+$\alpha$CMRU} \\
      \midrule
      Loss & \shortstack{1.443 \\ {\scriptsize [1.440; 1.447]}} & \shortstack{1.455 \\ {\scriptsize [1.450; 1.461]}} & \shortstack{1.504 \\ {\scriptsize [1.496; 1.511]}} & \shortstack{1.453 \\ {\scriptsize [1.449; 1.458]}} & \shortstack{1.442 \\ {\scriptsize [1.438; 1.446]}} & \textbf{\shortstack{1.441 \\ {\scriptsize [1.437; 1.445]}}} \\
      \bottomrule
    \end{tabular}
  \end{sc}
\end{small}
\end{table*}

\textbf{Cumulative updates narrow the performance gap between CMRU and linear RNNs.} On IMDb and sCIFAR10, LRU and minGRU achieve the highest accuracy, consistent with their fading memory mechanisms. However, the performance gap is small and narrows with increased depth. With $\varepsilon=1$ the gradient magnitude equals unity along both the sequence axis; the CMRU thus acts as a residual connection in time, making deeper stacks as easy to train as shallower ones~(\cref{sec:ssm_representation}). We note that sCIFAR10 exhibited substantial overfitting across all architectures, a phenomenon also reported in prior work~\cite{orvieto2023resurrectingrecurrentneuralnetworks}.

\textbf{Discrete updates show advantages on ListOps.} Both CMRU variants outperform the baselines at deeper architectures, achieving approximately 61\% accuracy (close to state-of-the-art) with $r \geq 3$ layers, while LRU performance remains limited across all configurations. This suggests that persistent memory and discrete state transitions provide useful primitives for tasks requiring exact preservation of intermediate results.

\textbf{Task-dependent architectural advantages emerge.} Fading-memory architectures (LRU, minGRU) excel when information must be smoothly integrated over time, while persistent-memory architectures (CMRU, $\alpha$CMRU) show strengths when information must be preserved through discrete operations. Neither family uniformly dominates. We discuss implications for architecture design in \cref{sec:discussion}.


\textbf{Persistent memory extends to harder long-range tasks.} On Pathfinder, both CMRU
variants achieve above 90\% accuracy while LRU and minGRU perform at chance level
($\approx$50\%), extending the pattern observed on ListOps. While larger fading-memory
networks can solve Pathfinder with hyperparameter tuning~\cite{orvieto2023resurrectingrecurrentneuralnetworks}, LRU and minGRU at this capacity ($r=3$, $d=32$)
fail entirely, suggesting that persistent memory provides a decisive efficiency advantage
for tasks requiring long-range state preservation.

\textbf{Persistent memory complements fading memory in language modeling.}
We evaluate on character-level language modeling on the Shakespeare dataset, reporting
cross-entropy loss on the test set (lower is better). Results are presented in
\cref{tab:lm_shakespeare}. CMRU and $\alpha$CMRU are competitive with minGRU and LRU,
demonstrating that persistent memory does not impair performance on tasks dominated by
smooth statistical structure. Hybrid architectures combining minGRU with CMRU or
$\alpha$CMRU layers achieve the best results, though the improvements are marginal at
this scale. The primary finding here is one of compatibility: persistent and fading memory
can be trained jointly without conflict. Evidence that such combinations can yield
substantial gains in learning efficiency and generalization comes from reinforcement
learning, where combining both memory types leads to significant improvements in
horizon generalization~\cite{Bakija2026MultistabilityRL}.
The CMRU and $\alpha$CMRU results are obtained using $\varepsilon$-annealing during training;
generalization to longer sequence lengths may require additional training strategies not explored in this work.

\subsection{Audio Classification at Minimal Capacity}

We now evaluate whether this competitiveness extends to \textit{practically relevant} audio classification tasks at the \textit{minimal capacity regime} required for ultra-low power deployment~\cite{Fyon2026AnalogRecurrent}. We assess two keyword spotting (KWS) tasks derived from the Google Speech Commands dataset~\cite{warden2018speechcommandsdatasetlimitedvocabulary}: digit recognition (11 classes) and full vocabulary classification (35 classes). These tasks represent canonical applications for \textit{always-on} audio processing in battery-constrained devices. We evaluate single-layer architectures ($r=1$) with state dimension $d\in \{4, 8, 16\}$, corresponding to easily implementable analog device sizes. Results are presented in \cref{tab:audio_tasks_r1}. We extend the analysis to two-layer architectures and wake-word detection (binary: ``yes'' versus other words) in \cref{appendix:additional_results}.

\textbf{All architectures achieve strong performance at minimal capacity.} On KWS Digits, all cells exceed 95\% accuracy even at $d=4$. On KWS All (35 classes), performance ranges from 88\% to 92\%, with consistent improvement as state dimension increases.

\begin{table}[htbp!]
\caption{
  Accuracy (\%) for audio classification (KWS Digits, KWS All) with $r=1$ layer and $d \in \{4, 8, 16\}$.
  Mean with min--max range over multiple initializations.
}
\label{tab:audio_tasks_r1}
\centering    \begin{small}
      \begin{sc}
        \setlength{\tabcolsep}{3pt}
        \begin{tabular}{lcccc}
          \toprule
          $(r, d)$ & $\alpha$CMRU & CMRU & LRU & minGRU \\
          \midrule
          \multicolumn{5}{c}{\textbf{KWS Digits}} \\
          \midrule
          (1, 4) & \shortstack{95.13 \\ {\scriptsize [94.75; 95.47]}} & \shortstack{95.12 \\ {\scriptsize [94.69; 95.50]}} & \shortstack{95.49 \\ {\scriptsize [94.62; 96.28]}} & \shortstack{95.92 \\ {\scriptsize [95.41; 96.31]}} \\
          (1, 8) & \shortstack{95.09 \\ {\scriptsize [94.50; 95.69]}} & \shortstack{95.40 \\ {\scriptsize [95.03; 95.78]}} & \shortstack{96.16 \\ {\scriptsize [95.47; 96.56]}} & \shortstack{96.05 \\ {\scriptsize [95.81; 96.53]}} \\
          (1, 16) & \shortstack{95.49 \\ {\scriptsize [95.22; 95.75]}} & \shortstack{95.38 \\ {\scriptsize [94.91; 95.91]}} & \shortstack{96.34 \\ {\scriptsize [95.88; 96.66]}} & \shortstack{95.99 \\ {\scriptsize [95.69; 96.25]}} \\
          \midrule
          \multicolumn{5}{c}{\textbf{KWS All}} \\
          \midrule
          (1, 4) & \shortstack{88.41 \\ {\scriptsize [86.69; 90.00]}} & \shortstack{88.98 \\ {\scriptsize [88.34; 89.59]}} & \shortstack{89.85 \\ {\scriptsize [86.81; 91.09]}} & \shortstack{90.12 \\ {\scriptsize [87.41; 91.91]}} \\
          (1, 8) & \shortstack{89.25 \\ {\scriptsize [88.19; 90.28]}} & \shortstack{88.99 \\ {\scriptsize [86.97; 90.03]}} & \shortstack{90.62 \\ {\scriptsize [88.81; 92.12]}} & \shortstack{91.82 \\ {\scriptsize [90.88; 92.88]}} \\
          (1, 16) & \shortstack{89.53 \\ {\scriptsize [88.69; 90.28]}} & \shortstack{89.76 \\ {\scriptsize [88.97; 90.44]}} & \shortstack{91.82 \\ {\scriptsize [90.81; 92.88]}} & \shortstack{91.13 \\ {\scriptsize [90.84; 91.53]}} \\
          \bottomrule
        \end{tabular}
      \end{sc}
    \end{small}
  \vskip -0.1in
\end{table}

\textbf{Performance gaps are small and consistent.} LRU and minGRU achieve marginally higher accuracy on most configurations, but cross-seed variability often exceeds inter-architecture differences.

\textbf{CMRU enables ultra-low power deployment without sacrificing accuracy.} CMRU achieves accuracy within 1--2\% of the best-performing baselines while retaining compatibility with ultra-low power analog implementation. Its quantized structure maps directly to analog primitives, enabling 100~nW operation~\cite{Fyon2026AnalogRecurrent}. 
A single-layer CMRU with $d=16$ achieves 95.38\% on digit recognition and 89.76\% on full vocabulary classification.

\section{Reflection Dynamics and Modular Computation}\label{subsec:parity}

While the CMRU with $\varepsilon = 1$ demonstrates strong performance across tasks requiring persistent memory and long-range dependencies, we now identify a class of problems that expose fundamental limitations shared by all architectures examined in this work. Analyzing these limitations reveals promising directions for future work on extending the expressivity of parallelizable RNNs.

We evaluate all architectures on a binary parity classification task: given an input sequence $x = (x_1, x_2, \ldots, x_L)$ where $x_t \in \{0, 1\}$, the model must predict $y = \left(\sum_{t=1}^{L} x_t\right) \bmod 2$. Models are trained on sequences of length $L \in [50, 400]$ and evaluated on $L \in [50, 1000]$ to assess length generalization, using last pooling with a single recurrent layer of dimension $d=1$.

The CMRU variants ($\varepsilon \geq 0$) do not successfully solve this task, nor do the LRU or minGRU baselines (\cref{tab:Parity_alphaCMRU}). For the CMRU with $\varepsilon = 1$, maintaining a cumulative sum is trivial. However, the modulo operation proves intractable for MLPs when generalizing to unseen sequence lengths, as the MLP must learn a periodic function over an unbounded domain. The root cause lies in the structure of reachable states: cumulative dynamics create countably infinite fixed points organized as a monotonic lattice, whereas the ideal solution for modular arithmetic requires states that alternate periodically. This mismatch suggests that different computational primitives are needed in the recurrent part to organize states in a way that makes the subsequent classification trivial.

\newpage
Negative $\varepsilon$ values enable the required reflection behavior. With $\varepsilon = -1$, the update equation becomes:
\begin{equation}
    h_t =  z_t \odot \left(\Sign{\hat{h}_t} \odot \alpha - h_{t-1}\right) + (1-z_t) \odot h_{t-1}\,.
\end{equation}
During updates ($z_t = 1$), this implements a reflection operation that creates flip dynamics organizing the state space into alternating regions corresponding to even versus odd parity. This organization makes the subsequent modulo-2 classification trivial for the MLP, as the hidden state explicitly encodes parity through its position rather than its magnitude. Evaluating the CMRU with $\varepsilon = -1$ yields perfect accuracy (100\%) with robust length generalization to $L=1000$ (\cref{tab:Parity_alphaCMRU}).

This finding directly aligns with recent work by \citet{grazzi2025unlockingstatetrackinglinearrnns}, which proves that state-tracking tasks (including parity) require transition matrices with eigenvalues having negative real parts. Our $\varepsilon = -1$ formulation instantiates this principle: as shown in \cref{sec:ssm_representation}, the CMRU admits a selective SSM representation where $\varepsilon$ directly controls the eigenvalues of the state transition matrix during updates. Setting $\varepsilon = -1$ yields eigenvalues at $-1$, enabling the reflection dynamics necessary for modular computation.


\section{Perspectives and Limitations}\label{sec:discussion}

\paragraph{Complementarity between linear RNNs and CMRU computations.}
Our experiments reveal complementary strengths: fading-memory RNNs such as LRU and minGRU excel on tasks requiring smooth temporal integration, while CMRU variants show advantages on tasks requiring persistent memory and discrete state preservation. This specialization reflects fundamental differences in dynamics, which could be exploited through the development of hybrid architectures that combine the strengths of linear RNNs and CMRUs.

\paragraph{Toward ultra-low power RNN hardware.}
Two deployment paths are available. The first uses $\varepsilon$-annealing during training, producing models structurally identical to the original BMRU ($\varepsilon=0$) that run directly on existing validated Schmitt trigger hardware~\cite{Fyon2026AnalogRecurrent} at sub-microwatt power levels; the annealing procedure is detailed in \cref{app:annealing}. The second natively implements $\varepsilon=1$ dynamics using a new analog circuit primitive. Circuit details, transistor-level schematics, and simulation results for both paths are provided in \cref{app:hardware}.

\paragraph{The role of $\boldsymbol{\varepsilon}$ and learnable interpolation.}
Experiments have shown that the ideal value of $\varepsilon$ is task-specific and connects to the role of eigenvalues in linear RNNs. Future work could focus on developing methods for automatic optimization of $\varepsilon$ to avoid task-specific parameter tuning.

\paragraph{Limitations and future directions.}
Our evaluation focuses on relatively small model sizes motivated by ultra-low power deployment; scaling behavior at higher capacities and on more complex benchmarks remains an open question. We did not perform task-specific hyperparameter tuning, which may underestimate peak performance for all architectures: our consistent protocol prioritizes fair comparison and robustness over absolute benchmark numbers.

\section{Conclusion}

In this work, we establish persistent memory as a reliable, trainable primitive for parallelizable RNNs. We identified gradient blocking in update mode as a key limitation of the BMRU and introduced a cumulative update formulation that restores gradient flow while preserving persistent memory. The resulting CMRU and $\alpha$CMRU show dramatically improved convergence stability, match fading-memory parallelizable RNNs on smooth integration benchmarks, and substantially outperform them on tasks requiring persistent memory: at minimal capacity, both variants exceed 90\% accuracy on Pathfinder while LRU and minGRU remain at chance level. The CMRU retains the quantized states and discrete dynamics essential for ultra-low power analog implementation.

The CMRU is structurally complementary to fading-memory cells: where LRU, SSMs, and parallelizable gated RNNs excel at smooth integration, the CMRU excels at indefinite retention. This complementarity is a feature rather than a limitation. On character-level language modeling, hybrid architectures combining minGRU with CMRU or $\alpha$CMRU layers achieve the best performance among the architectures we evaluated (\cref{tab:lm_shakespeare}), suggesting that pairing persistent and fading memory within the same network is a productive direction. More broadly, the CMRU is a general-purpose persistent-memory primitive: its associative-scan formulation and structural compatibility with existing SSM pipelines make it a drop-in persistent-memory block in Mamba-style architectures, complementing their fading-memory dynamics with the ability to retain discrete state indefinitely. Persistent memory is relevant wherever discrete state must be retained across arbitrary delays, independently of the power budget. Such memory is already a necessary condition for horizon generalization in reinforcement learning~\cite{Bakija2026MultistabilityRL}, and we anticipate it will prove equally essential in other settings that demand long-range discrete recall.

\clearpage

\section*{Acknowledgements}

J.B.\ and A.F.\ are, respectively, a FRIA grantee (FRIA-40038025) and a Research Fellow (ASP-REN40024838) of the Fonds de la Recherche Scientifique -- FNRS. This work is supported by the Belgian Government through the Federal Public Service Policy and Support. Computational resources have been provided by the Consortium des Équipements de Calcul Intensif (CÉCI), funded by the FNRS under Grant No.\ 2.5020.11 and by the Walloon Region. The present research also benefited from computational resources made available on Lucia, the Tier-1 supercomputer of the Walloon Region, infrastructure funded by the Walloon Region under grant agreement No. 1910247. We thank Pierre Sacré, Loris Mendolia, and Asad Bakija for their valuable help, and the reviewers for their thoughtful and constructive discussions.

\section*{Disclosure}
This work has been the subject of patent applications under numbers EP26175243.0 and EP26175248.9.

\section*{Impact Statement}

The deployment of AI systems faces a dual challenge. First, the computational demands of modern deep learning contribute meaningfully to global energy consumption and carbon emissions. Second, high energy requirements prevent AI deployment in settings where power is severely constrained, leaving important applications unreachable despite their potential societal value.
This work addresses both challenges by developing recurrent neural networks for ultra-low power deployment, achieving competitive performance at sub-microwatt power levels. By reducing energy consumption by three orders of magnitude compared to conventional implementations, this research both limits environmental impact and unlocks applications that were previously infeasible due to energy constraints.
We believe that one of the levers research can provide is the study of specialized solutions that reduce energy consumption without sacrificing accuracy. By treating efficiency as a first-order design constraint rather than an afterthought, we can simultaneously reduce the environmental footprint of AI and expand its reach to energy-constrained settings where it could not otherwise be deployed. Reducing per-inference energy does not automatically ensure positive outcomes, and deployment contexts require ongoing attention from researchers, developers, and policymakers.

\section*{Software and Data}

All code, raw training logs, tabulated results, and reproduction scripts are publicly available.\footnote{\url{https://github.com/julienbrandoit/ICML2026---Improving-Performance-and-Stability-of-Ultra-Low-Power-RNNs.git}} The implementation is provided in \texttt{JAX} \cite{jax2018github, flax2020github}, but the architectures and training procedures can be straightforwardly reproduced using any standard open-source machine learning library such as \texttt{PyTorch} or \texttt{TensorFlow}. The repository includes detailed instructions for reproducing all experiments presented in this paper.

All benchmark datasets used in this work are publicly available: MNIST \cite{lecun1998gradient}, CIFAR-10 \cite{krizhevsky2009learning}, IMDb \cite{maas2011imdb}, ListOps \cite{nangia2018listopsdiagnosticdatasetlatent}, Pathfinder \cite{tay2020longrangearenabenchmark}, Shakespeare \cite{pmlr-v54-mcmahan17a}, and Google Speech Commands \cite{warden2018speechcommandsdatasetlimitedvocabulary}. Synthetic tasks (copy first input, parity) are generated on-the-fly using the provided scripts.

\bibliography{bib/bib}
\bibliographystyle{sty/icml2026}
\clearpage
\appendix
\section{Related Work}\label{app:related_work}

\paragraph{Parallelizable recurrent architectures.}
The renewed interest in RNNs for sequence processing, particularly as alternatives to Transformers~\cite{vaswani2017attentionneed}, can largely be attributed to recent developments in structured state-space models (SSMs)~\cite{gu2021efficientlymodelinglongsequences, gu2022parameterizationinvariantstructuredstatespace}. The Mamba architecture~\cite{gu2024mambalineartimesequencemodeling} and its successor Mamba-2~\cite{dao2024transformersssmsgeneralizedmodels} have demonstrated that parallelizable recurrent models can achieve competitive performance with Transformers on many benchmarks. A fundamental theoretical gap exists between parallelizable architectures and classical nonlinear RNNs: recent work~\cite{merrill2023illusionstateinmodernrecurrent} has shown that both Transformers and linear diagonal SSMs are constrained to the complexity class $\mathrm{TC}^0$, while classical RNNs can express problems in $\mathrm{NC}^1$, which strictly contains $\mathrm{TC}^0$ under widely held complexity-theoretic conjectures. Importantly, our experiments demonstrate that even within $\mathrm{TC}^0$ (the class of problems both Transformers and classical SSMs can theoretically express), significant practical performance differences persist among parallelizable recurrent cells. We show that persistent memory capabilities, which enable robust gradient propagation and perfect information preservation, are key for improving performance on long-range tasks. Parallelizable gated RNNs have emerged as an alternative to SSMs, including xLSTM~\cite{beck2024xlstmextendedlongshortterm}, HGRN~\cite{qin2023hierarchicallygatedrecurrentneural}, MatMul-free LM~\cite{zhu2025scalablematmulfreelanguagemodeling}, and GILR~\cite{martin2018parallelizinglinearrecurrentneural}. We adopt the minimal GRU (minGRU)~\cite{feng2024weregrudminimalgatedrecurrent} as our representative baseline due to its simplicity as one of the most minimal fully parallelizable gated architectures.

\paragraph{Gradient propagation and memory in RNNs.}
In recurrent formulations where processing is inherently local, in contrast to attention mechanisms that enable global information flow, the implementation of persistent memory remains challenging. Persistent memory is intrinsically linked to gradient quality, which directly impacts both long-range modeling capabilities and training stability. While Transformers largely avoid vanishing and exploding gradient issues through their architecture, the return to recurrent formulations reintroduces these fundamental challenges~\cite{hochreiter1991untersuchungen, bengio1994learning}. The dynamical systems perspective on these phenomena~\cite{pascanu2013difficultytrainingrecurrentneural} provides valuable insights into the mechanisms underlying gradient pathologies.

\paragraph{Fading versus persistent memory.}
From a dynamical systems perspective, the distinction between fading memory and persistent memory is well-established~\cite{boyd1985fadingmemory}. Linear dynamics imply monostability: states decay exponentially toward a unique equilibrium, which is the signature of fading memory. True persistent memory requires multistability, which linear dynamics cannot exhibit. However, multistability is rarely discussed in the machine learning literature despite its fundamental importance for implementing persistent memory. Notable exceptions include recent work exploring multistability in neural systems~\cite{vecoven2021bioinspiredbistablerecurrentcell, lambrechts2023warminguprecurrentneural} and the original BMRU formulation~\cite{degeeter2026parallelizablememoryrecurrentunits}.

\paragraph{Unitary and orthogonal RNNs.}
The desirability of eigenvalues near or at unit magnitude for stable gradient propagation has long been recognized~\cite{pascanu2013difficultytrainingrecurrentneural}. One line of work focuses on initialization strategies that place eigenvalues near the unit circle across different timescales~\cite{le2015simplewayinitializerecurrent, gu2020hipporecurrentmemoryoptimal, talathi2015improvingperformancerecurrentneural}, often grounded in theoretical frameworks such as HiPPO theory~\cite{gu2020hipporecurrentmemoryoptimal}. The Linear Recurrent Unit (LRU)~\cite{orvieto2023resurrectingrecurrentneuralnetworks}, which we adopt as our primary baseline, exemplifies this approach with its exponentially parameterized eigenvalues initialized within $[r_{\min}, r_{\max}]$. An alternative approach maintains unitary eigenvalues as a structural constraint during optimization~\cite{arjovsky2016unitaryevolutionrecurrentneural, wisdom2016fullcapacityunitaryrecurrentneural}. Unitary transition matrices can be constructed through products of simple unitary blocks or by optimizing directly on the Stiefel manifold. Related approaches employ the Cayley transform~\cite{helfrich2018orthogonalrecurrentneuralnetworks} and enforce orthogonality through soft or hard constraints~\cite{vorontsov2017orthogonalitylearningrecurrentnetworks, lezcano2019cheaporthogonalityconstraintsdeep}. However, in these approaches, the transition matrices are learned but remain fixed with respect to the input during inference. Our CMRU formulation with $\varepsilon = 1$ achieves unit eigenvalues structurally in both retain and update modes (see \cref{sec:ssm_representation}), combining the benefits of unitary RNNs with the flexibility of modern selective mechanisms.

\paragraph{Input-dependent (selective) mechanisms.}
Modern SSMs have introduced the concept of selectivity, where model parameters vary as functions of the input~\cite{gu2024mambalineartimesequencemodeling}. This input-dependence enables the model to selectively filter and retain information based on context, significantly improving performance on tasks requiring context-aware processing. Our work builds upon this paradigm: as shown in \cref{sec:ssm_representation}, the CMRU admits a selective SSM representation where the state transition matrix $A(x_t)$ depends on the input through the gate $z_t$, enabling discrete switching between memory modes while maintaining exactly unitary eigenvalues during retention. Recent theoretical work by~\citet{grazzi2025unlockingstatetrackinglinearrnns} has demonstrated that state-tracking tasks (including parity) require transition matrices with eigenvalues having negative real parts. Our reflection dynamics with $\varepsilon = -1$ (\cref{subsec:parity}) directly instantiate this principle, achieving perfect parity classification where all other configurations fail.

\paragraph{Gating mechanisms and non-differentiable activations.}
Classical gating mechanisms in LSTMs~\cite{hochreiter1997longshorttermmemory} and GRUs~\cite{cho2014learningphraserepresentationsusing}, as well as modern parallelizable variants such as minGRU~\cite{feng2024weregrudminimalgatedrecurrent}, employ gates constrained to the open interval $(0, 1)$ through smooth sigmoid activations. Critically, these gates can only asymptotically approach the boundary values $0$ and $1$, never reaching them exactly. This design choice, likely stemming from the historical requirement for differentiable non-linearities, fundamentally limits the ability to implement truly persistent memory, as gates cannot achieve the exact values necessary for perfect information preservation. The CMRU employs the Heaviside step function and sign function to achieve exact binary gating. To enable gradient-based learning with these discontinuous activations, we draw on surrogate gradient methodology developed in the spiking neural networks community~\cite{neftci2019surrogategradientlearningspiking}. Surrogate gradients substitute a smooth approximation during the backward pass while maintaining the discrete operation during the forward pass, enabling us to leverage exact binary switching while maintaining trainability through standard backpropagation.

\paragraph{Cumulative and integrator dynamics.}
The concept of integrator dynamics in recurrent systems has received limited attention in the machine learning literature, despite its importance for tasks requiring unbounded accumulation. Classical leaky integrators from neuroscience~\cite{dayan2005theoreticalneuroscience} implement exponential decay combined with input integration, but suffer from fading memory limitations. Perfect integrators, characterized by eigenvalues of exactly unity, are theoretically capable of unbounded accumulation without decay. Our CMRU formulation with $\varepsilon = 1$ provides a mechanism for implementing selective integration through a structural constraint that enforces unitary eigenvalues during both retain and update modes. The resulting cell accumulates quantized inputs into the hidden state while maintaining perfect gradient flow ($\tfrac{\partial h_t}{\partial h_{t-1}} = 1$ in both modes), enabling robust learning of long-range dependencies.

\paragraph{Ultra-low power neural network hardware.}
The deployment of neural networks on resource-constrained devices has motivated significant research in efficient architectures~\cite{banbury2021mlperf, warden2019tinyml}. Most existing implementations, whether on microcontrollers or FPGAs, emulate neural computation in digital hardware~\cite{fedorov2019sparse, chang2015recurrent}. True analog implementations offer fundamentally different scaling properties: latency independent of network depth and energy consumption decoupled from arithmetic precision. The BMRU was specifically designed for hardware-software co-design: its bistable, quantized structure maps directly to Schmitt triggers, a well-characterized analog primitive requiring only transistors. Recent implementations have demonstrated $>$90\% accuracy at 100~nW power consumption for keyword spotting~\cite{fyon2026fullytunableultralowpower, Fyon2026AnalogRecurrent}, three orders of magnitude below digital alternatives. The CMRU with $\varepsilon = 1$ requires different circuit primitives (e.g., gated charge accumulators) but maintains noise resilience through discrete accumulated values. The $\varepsilon$-annealing strategy described in \cref{app:annealing} provides a path to deploy gradient-friendly training while retaining compatibility with existing Schmitt trigger hardware.

Recent work has also explored the hardware realization of parallelizable gated RNNs such as the minGRU in low power regimes. In particular,~\citet{billaudelle2025minimalistswitchedcapacitorcircuitsefficient} propose a mixed-signal implementation of a modified minGRU using switched-capacitor in-memory computing. Their design relies on aggressive quantization, discrete-time operation, and explicit clocking to orchestrate charge transfer, gating, and state updates. While this approach demonstrates impressive energy efficiency compared to fully digital accelerators, its energy consumption and latency remain fundamentally tied to the clock frequency and the required sequencing of sampling, sharing, and analog-to-digital conversion phases.

\section{The $\boldsymbol{\alpha}$CMRU Admits a Selective State-Space Representation} \label{sec:ssm_representation}

We demonstrate that the CMRU and $\alpha$CMRU cells admit a selective state-space representation with unique properties that enable robust persistent memory. This result applies to both the CMRU with fixed quantization scale $\alpha$ and its generalization, the $\alpha$CMRU with input-dependent scaling $\alpha(x_t)$.

Consider a non-linear state-space model (SSM) (or in the modern language of SSMs, a \textit{selective} SSM) defined by
\begin{align}
    h_t = A(x_t) h_{t-1} + B(x_t)x_t\,,
\end{align}
where $h_t \in \mathbb{R}^d$ is the hidden state, $A(x_t) \in \mathbb{R}^{d \times d}$ denotes the state transition matrix and $B(x_t) \in \mathbb{R}^{d \times m}$ denotes the input-to-state mapping, with both matrices parameterized as functions of the input $x_t \in \mathbb{R}^{m}$. Throughout, we treat all vectors as column vectors, so that $x^T_t\in \mathbb{R}^{1 \times m}$ denotes the corresponding row vectors. While this model is nonlinear due to the input-dependence of $A$ and $B$, it remains linear with respect to the hidden state $h_t$, which permits the application of parallel scan algorithms for efficient computation of hidden states. Representative examples of such models include Mamba \cite{gu2024mambalineartimesequencemodeling} and Mamba-2 \cite{dao2024transformersssmsgeneralizedmodels}.

A fundamental limitation of SSMs concerns the implementation of persistent memory as it would require the eigenvalues of the state transition matrix $A$ to be on the unitary circle. Any deviation from this condition results in either vanishing or exploding memory over extended sequences, with the magnitude of deviation scaling as $\lambda^{\Delta t^*}$, where $\lambda$ denotes an eigenvalue and $\Delta t^* \in \mathbb{N}$ represents the number of time steps elapsed since the last informative input.

\begin{proposition}
The CMRU and $\alpha$CMRU cells can be expressed as a non-linear SSM wherein the state transition matrix $A$ maintains unitary eigenvalues during the retain mode, thereby enabling robust persistent memory.
\end{proposition}

\begin{proof}
We rewrite the CMRU (and $\alpha$CMRU) update equation as
\begin{align}
    h_t = \text{diag}(1 - z_t + \varepsilon z_t) h_{t-1} + z_t \odot \Signalpha{\hat{h}_t}\,,
\end{align}
where $\text{diag}(\cdot)$ constructs a square diagonal matrix from its vector argument, and where $\Signalpha{\cdot}$ is a short notation for $\alpha\odot \Sign{\cdot}$ or $\alpha_t \odot \Sign{\cdot}$, such that we handle both the CMRU and the $\alpha$CMRU through a single notation. The input-driven term can be decomposed as $z_t \odot \Signalpha{\hat{h}_t} = B(x_t) x_t$. We note that $B(x_t)$ is not uniquely defined. Indeed, consider the general form
\begin{align}
    B(x_t) = B_{0}(x_t) + C\,,
\end{align}
where $C \in \mathbb{R}^{d \times m}$ is an arbitrary matrix whose rows lie in the null space of $x_t$ (i.e., $C x_t = 0$), and
\begin{align}
    B_{0}(x_t) = (z_t \odot \Signalpha{\hat{h}_t}) \frac{x^T_t}{x^T_tx_t}\,.
\end{align}
Then,
\begin{align}
    \left(B_{0}(x_t) + C\right) x_t &= B_{0}(x_t)x_t + Cx_t \nonumber\\
    &= B_{0}(x_t)x_t \nonumber\\
    &= (z_t \odot \Signalpha{\hat{h}_t}) \frac{x^T_t}{x^T_tx_t} x_t \nonumber\\
    &= (z_t \odot \Signalpha{\hat{h}_t}) \frac{x^T_tx_t}{x^T_tx_t} \nonumber\\
    &= z_t \odot \Signalpha{\hat{h}_t}\,,
\end{align}
establishing the desired decomposition. We adopt the canonical form with $C = 0$, yielding the following non-linear SSM representation:
\begin{align}
    A(x_t) &= \text{diag}(1 - z_t + \varepsilon z_t)\,, \\
    B(x_t) &= (z_t \odot \Signalpha{\hat{h}_t}) \frac{x^T_t}{x^T_tx_t}\,.
\end{align}
This representation holds for both the CMRU (with fixed quantization scale $\alpha$) and the $\alpha$CMRU (with input-dependent $\alpha(x_t)$), as the scaling parameter is absorbed into $\Signalpha{\hat{h}_t}$ within the $B(x_t)$ matrix.
\end{proof}

\begin{remark}
During the retain mode, characterized by $z_t = 0$, the state transition matrix reduces to the identity:
\begin{align}
    \left.A(x_t)\right|_{z_t = 0} = I\,,
\end{align}
which possesses unit eigenvalues exclusively. This interpretation of the CMRU and $\alpha$CMRU as a special case of non-linear SSM with unitary eigenvalues during the retain mode provides a theoretical justification for their ability to model long range dependencies effectively. Unitary eigenvalues combined with the hysteresis non-linearity ensure that \emph{information can be preserved indefinitely and robustly in the hidden state}. The robustness is ensured by the hysteretic window, which acts as a noise filter that prevents small perturbations from spuriously triggering gate transitions and corrupting the stored memory.

During the update mode, characterized by $z_t = 1$, the state transition matrix reduces to the diagonal matrix:
\begin{align}
    \left.A(x_t)\right|_{z_t = 1} = \text{diag}(\varepsilon)\,.
\end{align}
Interestingly, it means that for $\varepsilon = 1$, the state transition matrix reduces to identity for both the retain and the update mode. This has an important consequence for deep architectures: gradient magnitude equals unity simultaneously along the sequence axis (through time) and along the depth axis (across stacked layers), since each recurrent block contributes a unit Jacobian in both modes. The CMRU with $\varepsilon = 1$ therefore acts as a residual connection in time, making training of deep stacks as stable as training shallow ones. We can also observe that the parity example (\cref{subsec:parity}), requiring at least one negative eigenvalue, can be directly related to the recent work of \cite{grazzi2025unlockingstatetrackinglinearrnns} under the matrix-vector transition formalization.
\end{remark}

\subsection{Fixed Points in Retain Mode}

In retain mode ($z_t = 0$), setting $z_t = 0$ in \cref{eq:cmru} gives $h_t = h_{t-1}$
directly. Every $h \in \mathbb{R}^d$ is therefore a fixed point of the update equation,
for any value of $\varepsilon$: the CMRU retains any state indefinitely in the absence
of a sufficiently strong input, which is the formal expression of persistent memory.
The hysteretic threshold $\beta_t$ controls how strong an input must be to exit retain
mode and trigger an update.

\section{Optimal Quantization on \texorpdfstring{$[-1,1]$}{[-1,1]} and Minimal MAE} \label{app:quantized}

\subsection{Theoretical Derivations}

We consider the problem of encoding a scalar real value $x_0 \in [-1,1]$ using $b=4$ binary memory slots, and we derive the minimal achievable mean absolute error (MAE) under an optimal fixed-rate quantization scheme. This encoding problem is associated with the \emph{copy first input (continuous)} tasks. While this result is well known in classical quantization theory, we present it here to establish the absolute lower bound of error that any quantized learning cell can achieve, thereby providing a fundamental limit against which to evaluate the (quantized) BMRU.

For a uniform source distribution on a bounded interval, the optimal scalar quantizer (for both MAE and MSE) is a uniform quantizer with midpoint reconstruction \cite{quantization}. With $B = 2^b$ quantization levels, the interval $[-1,1]$ is partitioned into bins of equal width:
\begin{equation}
    \Delta = \frac{2}{B}\,.
\end{equation}
Each bin is associated with its midpoint as the reconstruction value. The quantization error $e(x_0) = \hat{x}_0 - x_0$ therefore lies in the interval $[-\Delta/2, \Delta/2]$. Conditioned on a given bin, $e(x_0)$ is uniformly distributed on this interval, and its expected absolute value is:
\begin{equation}
    \mathbb{E}\left[|e(x_0)|\right]
    = \frac{1}{\Delta} \int_{-\Delta/2}^{\Delta/2} |u| \, du
    = \frac{\Delta}{4}\,.
\end{equation}
Since all bins have equal probability under the uniform distribution, this value is also the global MAE. Substituting $\Delta$ yields:
\begin{equation}
    \mathcal{E}_{\mathrm{MAE}}^{\star} = \frac{1}{2^{b+1}}\,.
\end{equation}
Although the interval is symmetric around zero and thus includes signed values, the sign of $x_0$ does not require a dedicated bit. The sign is implicitly encoded by the quantization index itself, as the quantization levels are distributed across both negative and positive values.

\subsection{Extensions to the Cumulative Variants}

At first glance, one might expect the optimal quantization bound $\mathcal{E}_{\mathrm{MAE}}^{\star}$ to not apply to the cumulative variants of the BMRU (i.e., when $\varepsilon \neq 0$). Indeed, the cumulative formulation allows the hidden state to accumulate values over multiple timesteps, potentially enabling the state to be pushed toward any multiple of the quantization level $\alpha$. This suggests that, in principle, a CMRU ($\varepsilon \neq 0$) could integrate information over time to progressively refine its encoding and achieve representations that transcend the discrete $2^b$ levels available when $\varepsilon = 0$.

However, two properties of our specific task and architecture prevent this theoretical advantage from being realized in practice:

\textbf{(1) State-independent updates.} In all the recurrent cells we consider, the relative update to the hidden state at each timestep is computed independently of the previous state value. This architectural constraint means that the cell cannot adaptively modulate its updates based on the current encoding error or state value with a single-block architecture.

\textbf{(2) Single-timestep information presentation.} In the copy first input task, the information to be encoded is presented at a single timestep ($t=0$). This structure eliminates any opportunity for the cell to integrate information across multiple timesteps to progressively refine its approximation.

These two conditions together ensure that, despite the theoretical flexibility of cumulative state updates when $\varepsilon \ne 0$, the CMRU remains subject to the same fundamental quantization constraint.

\section{Training with $\boldsymbol{\varepsilon}$ Annealing for Hardware Deployment}\label{app:annealing}

Our experiments establish that $\varepsilon = 1$ provides superior gradient flow, convergence stability, and reduced initialization sensitivity. However, existing ultra-low power analog implementations are designed for $\varepsilon = 0$ (the original BMRU), where the architecture maps directly onto Schmitt trigger circuits~\cite{Fyon2026AnalogRecurrent}. The CMRU with $\varepsilon = 1$ requires alternative circuit primitives such as gated charge accumulators. We explore an intermediate path: training with $\varepsilon = 1$ to leverage optimal gradient propagation, then gradually decaying $\varepsilon$ toward zero to produce models compatible with existing hardware.

\textbf{Annealing protocol.} We interpolate $\varepsilon$ from 1 to 0 during training according to a piecewise linear schedule: $\varepsilon$ remains at 1 for the first 5\% of training, decays linearly to 0 over the next 70\%, then stays at 0 for the final 25\%. This stationary phase addresses the theoretical concern that time-varying $\varepsilon$ renders the optimization problem non-stationary. We only save checkpoints after $\varepsilon = 0$ is reached, ensuring reported performance reflects the deployable configuration.

\textbf{Preliminary results.} \cref{tab:part4_sMNIST_decaying} presents accuracy on sMNIST across state dimensions $d \in \{4, 8, 16, 32, 64, 128, 256\}$ for three regimes: constant $\varepsilon = 0$ (i.e., the original BMRU), constant $\varepsilon = 1$, and decaying $\varepsilon$. For the quantized CMRU, annealing substantially improves over constant $\varepsilon = 0$ (88\% versus 30\% at $d=32$) but does not fully recover constant $\varepsilon = 1$ performance (96\% at $d=32$). The gap is most pronounced at small dimensions where discrete state constraints dominate. For the $\alpha$CMRU, annealing closely matches constant $\varepsilon = 1$ across nearly all dimensions (97\% versus 97\% at $d=32$, 98\% versus 97\% at $d=64$), suggesting continuous states successfully exploit gradient flow during training even when deployment requires $\varepsilon = 0$ dynamics.

These results demonstrate that $\varepsilon$ annealing provides a viable path to deploy gradient-friendly training on existing ultra-low power hardware, though quantized architectures at minimal capacity may require $\varepsilon$-aware circuit (or $\alpha$-aware circuit) implementations to fully realize performance gains.

\begin{table*}[t!]
  \caption{Accuracy (\%) for sMNIST with decaying epsilon across state dimensions. Note that $\varepsilon = 0$ corresponds to the original BMRU.}
  \label{tab:part4_sMNIST_decaying}
  \begin{center}
    \begin{small}
      \begin{sc}
        \begin{tabular}{lccccccc}
          \toprule
          & \multicolumn{7}{c}{State Dimension} \\
          \cmidrule(lr){2-8}
          & 4 & 8 & 16 & 32 & 64 & 128 & 256 \\
          \midrule
          \shortstack{CMRU \\ {\scriptsize decay $\varepsilon$}} &
          \shortstack{30.29 \\ {\scriptsize [20.94; 47.72]}} &
          \shortstack{51.68 \\ {\scriptsize [32.41; 62.34]}} &
          \shortstack{79.12 \\ {\scriptsize [66.75; 85.19]}} &
          \shortstack{87.94 \\ {\scriptsize [81.88; 92.53]}} &
          \shortstack{95.46 \\ {\scriptsize [95.00; 96.44]}} &
          \shortstack{96.78 \\ {\scriptsize [96.31; 97.22]}} &
          \shortstack{97.03 \\ {\scriptsize [96.62; 97.56]}} \\
          \shortstack{CMRU \\ {\scriptsize $\varepsilon = 1$}} &
          \shortstack{83.61 \\ {\scriptsize [80.97; 85.09]}} &
          \shortstack{91.49 \\ {\scriptsize [90.53; 92.69]}} &
          \shortstack{94.83 \\ {\scriptsize [93.12; 95.84]}} &
          \shortstack{95.92 \\ {\scriptsize [95.41; 96.19]}} &
          \shortstack{97.05 \\ {\scriptsize [96.69; 97.28]}} &
          \shortstack{96.20 \\ {\scriptsize [95.16; 97.00]}} &
          \shortstack{96.16 \\ {\scriptsize [95.81; 96.38]}} \\
          \shortstack{CMRU \\ {\scriptsize $\varepsilon = 0$}} &
          \shortstack{14.52 \\ {\scriptsize [11.69; 18.62]}} &
          \shortstack{17.09 \\ {\scriptsize [11.69; 21.41]}} &
          \shortstack{21.76 \\ {\scriptsize [18.69; 25.09]}} &
          \shortstack{35.32 \\ {\scriptsize [28.12; 44.41]}} &
          \shortstack{47.49 \\ {\scriptsize [46.16; 49.38]}} &
          \shortstack{54.82 \\ {\scriptsize [48.62; 73.69]}} &
          \shortstack{80.11 \\ {\scriptsize [76.09; 85.34]}} \\
          \midrule
          \shortstack{$\alpha$CMRU \\ {\scriptsize decay $\varepsilon$}} &
          \shortstack{88.13 \\ {\scriptsize [87.28; 89.41]}} &
          \shortstack{93.82 \\ {\scriptsize [93.22; 94.59]}} &
          \shortstack{96.12 \\ {\scriptsize [95.59; 96.88]}} &
          \shortstack{96.94 \\ {\scriptsize [96.22; 97.28]}} &
          \shortstack{97.73 \\ {\scriptsize [97.69; 97.84]}} &
          \shortstack{98.09 \\ {\scriptsize [97.88; 98.25]}} &
          \shortstack{98.07 \\ {\scriptsize [97.88; 98.25]}} \\
          \shortstack{$\alpha$CMRU \\ {\scriptsize $\varepsilon = 1$}} &
          \shortstack{88.34 \\ {\scriptsize [87.38; 89.53]}} &
          \shortstack{93.49 \\ {\scriptsize [91.56; 94.62]}} &
          \shortstack{95.43 \\ {\scriptsize [93.03; 96.47]}} &
          \shortstack{96.76 \\ {\scriptsize [96.22; 97.00]}} &
          \shortstack{96.88 \\ {\scriptsize [96.66; 97.12]}} &
          \shortstack{97.16 \\ {\scriptsize [96.72; 97.91]}} &
          \shortstack{97.25 \\ {\scriptsize [96.69; 98.06]}} \\
          \shortstack{$\alpha$CMRU \\ {\scriptsize $\varepsilon = 0$}} &
          \shortstack{32.79 \\ {\scriptsize [24.62; 42.94]}} &
          \shortstack{40.78 \\ {\scriptsize [38.38; 42.75]}} &
          \shortstack{44.76 \\ {\scriptsize [38.56; 55.16]}} &
          \shortstack{48.86 \\ {\scriptsize [43.12; 64.84]}} &
          \shortstack{69.60 \\ {\scriptsize [66.12; 75.69]}} &
          \shortstack{88.30 \\ {\scriptsize [83.78; 94.59]}} &
          \shortstack{94.06 \\ {\scriptsize [92.06; 95.28]}} \\
          \bottomrule
        \end{tabular}
      \end{sc}
    \end{small}
  \end{center}
  \vskip -0.1in
\end{table*}

\section{Hardware Implementation}\label{app:hardware}

\subsection{Deployment via $\boldsymbol{\varepsilon}$-Annealing on Existing Hardware}

The $\varepsilon$-annealing strategy (detailed in \cref{app:annealing}) trains with $\varepsilon=1$
for optimal gradient flow, then decays $\varepsilon$ linearly to 0. The resulting model is
structurally identical to the original BMRU and deploys on existing Schmitt trigger circuits
without any new circuit primitives, achieving sub-microwatt operation as validated in
\citet{Fyon2026AnalogRecurrent}.

\subsection{Proof-of-Concept Native CMRU Circuit}

The native CMRU with $\varepsilon=1$ requires a different analog primitive. We propose a
cascade of BMRU cells interleaved with Differential Pair Integrator (DPI) blocks. Each DPI
integrates toward $\alpha$ when the corresponding BMRU output is high, and the hysteresis
windows are shifted by $\alpha$ at each stage to implement progressive triggering. The output
is the sum of all BMRU outputs, implementing the cumulative state update. All parameters are
tunable via bias currents, and both BMRU and DPI blocks operate in the subthreshold regime.

\begin{figure*}[htbp!]
  \vskip 0in
  \begin{center}
    \centerline{\includegraphics[width=0.75\textwidth]{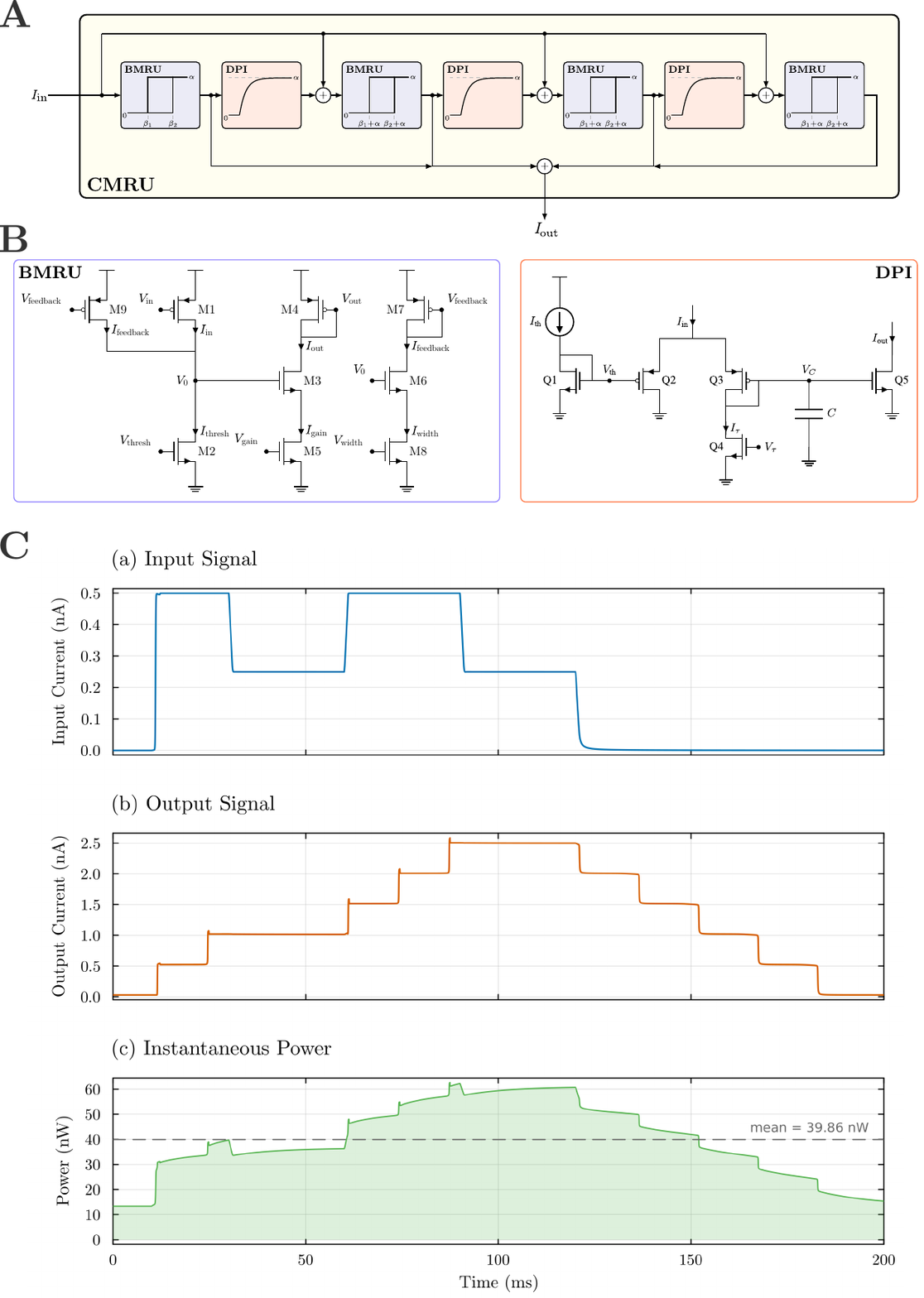}}
    \caption{
          Native CMRU analog circuit.
          \textbf{(A)} Block diagram of the conceptual architecture. The CMRU is realized as a cascade of BMRU cells~\cite{Fyon2026AnalogRecurrent} interleaved with Differential Pair Integrator (DPI) blocks. Each BMRU stage occupies one level of the fixed-point ladder, and each DPI integrates the output of the preceding stage toward a multiple of $\alpha$, shifting the hysteresis window of the next stage. The finite set of reachable fixed points, located at integer multiples of $\alpha$, matches the discrete cumulative state values that naturally emerge from training.
          \textbf{(B)} Transistor-level schematics of the BMRU cell (left) and of the DPI block (right). Both blocks operate in the subthreshold regime.
          \textbf{(C)} Transistor-level Cadence simulation results. The high threshold of the BMRUs is set to \SI{0.35}{\nano\ampere}, the low threshold to \SI{0.15}{\nano\ampere}, $\alpha$ to \SI{0.5}{\nano\ampere}, and the DPI time constant to \SI{10}{\milli\second}. Top: time evolution of the input signal. Depending on amplitude, the input either drives the memory to one of its two states (high or low) or, at the retain level of \SI{0.25}{\nano\ampere}, leaves the stored value unchanged. Middle: time evolution of the output signal. In set mode, the output is incremented (or decremented) by $\alpha$ every \SI{10}{\milli\second}; in retain mode, the output is held indefinitely. Bottom: time evolution of the instantaneous power. Consumption scales with the output level but remains deep in the ultra-low-power regime, well below the
          microwatt range.
    }
    \label{fig:cmru_hardware}
  \end{center}
\end{figure*}

Transistor-level schematics and Cadence simulations (\cref{fig:cmru_hardware}) confirm
ultra-low power operation at 39.86~nW. Circuit area, power, and mismatch robustness
optimization are left for future work.

\section{Architecture and Training Details} \label{app:training}

\subsection{Backbone Architecture Details} \label{subsection:backbone}

All experiments are performed by only changing the type of recurrent cell in the backbone illustrated in \cref{fig:backbone}. In all experiments, the model dimension is fixed to $m = 256$. The model dimension $m$ refers to the dimensionality of sequences between layers, while the state dimension $d$ refers to the dimensionality of the recurrent cell hidden state, and $r$ to the number of blocks. Note that $m$ and $d$ are independent: the recurrent cell projects from $m$ to $d$ internally, and projects back to $m$ for the output.

\begin{figure}[htbp!]
  \vskip 0in
  \begin{center}
    \centerline{\includegraphics[width=0.6\columnwidth]{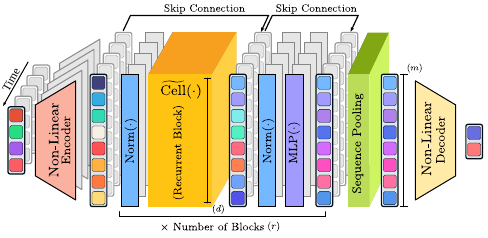}}
    \caption{
      Common backbone architecture used across all experiments. The architecture interleaves recurrent layers with MLPs, skip connections, and normalization. We study different recurrent cell types (CMRU, $\alpha$CMRU, LRU, minGRU) by varying the internal mechanism within $\CellB{\cdot}$ while keeping the backbone structure fixed. The model dimension is fixed at $m = 256$ across all experiments. The state dimension $d$ refers to the dimensionality of the recurrent cell hidden state, and $r$ to the number of blocks.
    }
    \label{fig:backbone}
  \end{center}
\end{figure}

\subsection{Cell Output and Integration}

Each cell defines an internal state transition law that maps the input sequence to a hidden state sequence $h_{1:T} \in \mathbb{R}^{d\times T}$ via $h_{1:T} = \Cell{x_{1:T}; h_0}$, where $h_0$ represents the initial state and $d$ denotes the state dimension. All architectures considered in this work are parallelizable through associative scan algorithms \cite{blelloch1989scans, martin2018parallelizinglinearrecurrentneural, smith2022simplifiedstatespacemodels}, enabling efficient computation of hidden states across the entire sequence.

From the hidden sequence, the output sequence is generated through an input-gated projection of the normalized hidden sequence:
\begin{equation}\label{eq:cell_output}
\begin{aligned}
y_{1:T}
  &= \CellB{x_{1:T}} \\
  &= \Norm{\Linear{\Cell{x_{1:T}; h_0}}} \\
  &\quad \odot \sigma(\Linear{x_{1:T}})\,,
\end{aligned}
\end{equation}
where $\sigma(\cdot)$ is the element-wise sigmoid function, $\Linear{\cdot}$ denotes a learned linear transformation, and $\Norm{\cdot}$ represents layer normalization. The normalization ensures that even integrator cells ($\varepsilon = 1$) output sequences whose magnitudes do not explode, maintaining proper scale for subsequent MLPs and blocks. The input-dependent gate $\sigma(\Linear{x_{1:T}})$ ensures that each cell can learn to decouple information integration from information usage over the sequence. The complete block described by \cref{eq:cell_output} constitutes a sub-layer $\SubLayer{\cdot}$ that is integrated into the broader backbone architecture.

\subsection{MLP and Normalization Layers}

We use $\sigma(\cdot) = \GLU{\cdot}$ \cite{dauphin2017languagemodelinggatedconvolutional} within each MLP. The MLPs are defined as
\begin{align*}
    \MLP{x} = \Linear{\Dropout{\sigma(\Linear{x})}},
\end{align*}
with the input and output dimension equal to the model dimension ($m = 256$) and a hidden dimension of $4\times m$. The MLPs are applied point-wise, i.e., independently on each element in the sequence.

Normalization and skip-connection layers are applied point-wise and following a pre-norm scheme, i.e.,
\begin{align*}
    y = \upsilon \odot x + \SubLayer{\Norm{x}},
\end{align*}
where $\upsilon \in \mathbb{R}^{m}$ is a vector of learnable parameters, initialized to unity. The normalization used is LayerNorm \cite{ba2016layernormalization}. In this formulation, $\SubLayer{\cdot}$ can be $\MLP{\cdot}$ or $\CellB{\cdot}$ as described in \cref{eq:cell_output}.

\subsection{Sequence Pooling}

The backbone architecture must reduce the variable-length processed sequence to a fixed-size representation for prediction. We employ two pooling strategies depending on the task requirements:

\textbf{Last pooling} uses only the final timestep of the processed sequence for prediction. Formally, given the output sequence $y_{1:T} \in \mathbb{R}^{m \times T}$ from the final recurrent layer, the pooled representation is $y_{\text{pool}} = y_T \in \mathbb{R}^m$. This strategy is particularly suited for tasks requiring long-term memory, as it tests both the retention of information across the entire sequence and the ability to learn from long-range dependencies when supervision occurs only at sequence end.

\textbf{Mean pooling} averages the output sequence across all timesteps. The pooled representation is computed as $y_{\text{pool}} = \frac{1}{T}\sum_{t=1}^{T} y_t \in \mathbb{R}^m$. This strategy aggregates information from the entire sequence, making it more robust to local variations and suitable for tasks where relevant information is distributed throughout the sequence rather than concentrated at specific positions.

The pooled representation $y_{\text{pool}}$ is then passed through a final decoder to produce task-specific predictions.

\subsection{Positional Encoding}

All recurrent cells receive explicit positional information through sinusoidal positional encoding (PE) \cite{vaswani2017attentionneed}. The PE is concatenated to the input at each timestep and passed through a learned linear projection before being fed to the recurrent cell. Formally, for input $x_t \in \mathbb{R}^{m}$ at position $t$, the cell receives:
\begin{equation}
    \tilde{x}_t = \Linear{[x_t; \text{PE}(t)]} \in \mathbb{R}^m\,,
\end{equation}
where $\text{PE}(t) \in \mathbb{R}^{d_{\text{PE}}}$ denotes the sinusoidal positional encoding at position $t$, and $[\cdot;\cdot]$ denotes concatenation.

This design allows architectures that do not inherently require positional information (such as the LRU) to learn to suppress it through the projection if it acts as harmful noise, while enabling cells that benefit from explicit position awareness to leverage it when beneficial. Positional encoding has been shown to be necessary for the original BMRU with $\varepsilon = 0$ \cite{degeeter2026parallelizablememoryrecurrentunits}, where the lack of inherent temporal dynamics makes position information critical for sequence processing, but also for single-block minGRU architectures \cite{feng2024weregrudminimalgatedrecurrent}.

\subsection{Input Encoding and Output Decoding}

Raw task-specific inputs and outputs are transformed to and from the model dimension $m$ through nonlinear encoding and decoding layers.

The input encoder transforms task-specific input $\hat{x}_t \in \mathbb{R}^{d_{\text{task}}}$ to the model dimension:
\begin{equation}
    x_t = \tilde{x}_t + \MLP{\tilde{x}_t}, \quad \text{where } \tilde{x}_t = \Linear{\hat{x}_t}\,.
\end{equation}

Similarly, the output decoder transforms from model dimension back to task-specific output dimension $d_{\text{out}}$:
\begin{equation}
    \hat{y} = \tilde{y} + \MLP{\tilde{y}}, \quad \text{where } \tilde{y} = \Linear{y_\text{pool}}\,,
\end{equation}
where $y_{\text{pool}} \in \mathbb{R}^m$ is the pooled representation.

This nonlinear encoding and decoding design ensures that single-layer recurrent architectures are not limited by insufficient point-wise nonlinear computation capacity, but rather by the recurrent mechanism itself. This allows for fair architectural comparisons where observed performance differences can be attributed to the recurrent cell design rather than to inadequate input/output transformations.

\subsection{Recurrent Cell Implementations}

This section provides detailed specifications of the recurrent cells evaluated in our experiments. For the CMRU and $\alpha$CMRU architectures, we refer to \cref{eq:bmru_candidate,eq:bmru_threshold,eq:bmru_gate,eq:cmru,eq:alpha_cmru} in the main text.

\subsubsection{Linear Recurrent Unit (LRU)} \label{subsection:lru}

The LRU \cite{orvieto2023resurrectingrecurrentneuralnetworks} implements a diagonal complex-valued state-space model with strictly decaying exponentially parameterized eigenvalues for numerical stability. The state update follows:
\begin{align}
    x_t &= \Lambda \odot x_{t-1} + B u_t\,, \\
    y_t &= \mathrm{Re}(C x_t) + D u_t\,,
\end{align}
where $x_t \in \mathbb{C}^d$ is the complex-valued hidden state, $u_t \in \mathbb{R}^m$ is the input, $y_t \in \mathbb{R}^d$ is the output (before projection), and $\Lambda \in \mathbb{C}^d$ is a diagonal matrix of eigenvalues. The eigenvalues are parameterized as:
\begin{equation}
    \Lambda = \exp(-\exp(\nu) + i\exp(\theta))\,,
\end{equation}
where $\nu, \theta \in \mathbb{R}^d$ are learned parameters initialized such that eigenvalue magnitudes lie in $[r_{\min}, r_{\max}]$ with uniformly distributed phases in $[0, 2\pi]$. The input matrix $B \in \mathbb{C}^{m \times d}$ is normalized by $\gamma = \sqrt{1 - |\Lambda|^2}$ to maintain consistent input scaling across different eigenvalue magnitudes. The output matrices $C \in \mathbb{C}^{d \times d}$ and $D \in \mathbb{R}^{m \times d}$ are learned parameters.

The LRU is fully parallelizable via associative scan over the linear recurrence.

\subsubsection{Minimal Gated Recurrent Unit (minGRU)} \label{subsection:mingru}

The minGRU \cite{feng2024weregrudminimalgatedrecurrent} simplifies the traditional GRU by removing hidden state dependencies from the gating mechanism, enabling full parallelization while maintaining competitive performance. The update rule is:
\begin{align}
    z_t &= \sigma(W_z x_t + b_z)\,, \\
    \tilde{h}_t &= W_h x_t + b_h\,, \\
    h_t &= (1 - z_t) \odot h_{t-1} + z_t \odot \tilde{h}_t\,,
\end{align}
where $h_t \in \mathbb{R}^d$ is the hidden state, $x_t \in \mathbb{R}^m$ is the input, $z_t \in [0,1]^d$ is the update gate, $\tilde{h}_t \in \mathbb{R}^d$ is the candidate hidden state, and $\sigma(\cdot)$ denotes the sigmoid function.

A critical architectural property is that the sigmoid gate ensures $z_t \in (0, 1)$ (open interval), meaning the values $z_t = 0$ and $z_t = 1$ are only approached asymptotically but never attained. Consequently, the minGRU lacks a true retain mode where $z_t = 0$ would preserve the state unchanged, and lacks a true reset mode where $z_t = 1$ would completely overwrite the state. At every timestep, both the previous state and the candidate contribute to the update, preventing indefinite information preservation.

The update equation is a linear recurrence in $h_t$ and can be parallelized via associative scan.

\subsection{Training Generalities}

All experiments presented in this work follow a similar training procedure with hyperparameters that are not optimized per task. The objective is to observe (and encourage) the robustness of the cells to the use of standard hyperparameter values, while acknowledging that hyperparameter optimization could improve absolute performance metrics. We refer readers to \cref{section:protocol} for the motivations.

Optimization is performed with AdamW \cite{loshchilov2019decoupledweightdecayregularization} with $\beta_1 = 0.9$, $\beta_2 = 0.99$ and $\epsilon = 10^{-8}$ for all experiments. The weight decay (WD) value used is $0.0001$. The learning rate is set to $10^{-3}$ and follows a common scheduler across all experiments: (1) a linear warmup phase over $1\%$ of iterations starting from 0 up to the value of $10^{-3}$; followed by (2) a cosine decay to the minimum value of $10^{-5}$ spread over all remaining iterations. The gradient norm is clipped to 1 during training.

We adopt the terminology of iterations or, equivalently, epochs in the sense of a single gradient descent step throughout this work. The maximum number of epochs varies from task to task to ensure convergence within the limits of a reasonable computational budget (see \cref{tab:training_budget}). For each task, the current architecture is evaluated every 64 iterations on 20 batches sampled from the validation set. The checkpoint achieving the best performance during this evaluation is the one saved and used during the final evaluation procedure on the test set. Early stopping is performed if the performance is optimal during 100 consecutive evaluations (that is, 100\% accuracy). A batch size of 64 is used.

\begin{table*}[t!]
\caption{
  Maximum number of training iterations (epochs) allocated per task.
  All tasks use a batch size of 64, with evaluation performed every 64 iterations on 20 validation batches.
  Early stopping is triggered if optimal performance (100\% accuracy for classification tasks) is maintained for 100 consecutive evaluations.
}
\label{tab:training_budget}
  \begin{center}
    \begin{small}
      \begin{sc}
        \setlength{\tabcolsep}{3pt}
        \begin{tabular}{lcccccccccc}
          \toprule
          Task & \shortstack{Copy\\First} & sMNIST & pMNIST & \shortstack{sCIFAR\\-10} & ListOps & \shortstack{Path\\finder} & LM & IMDb & KWS & Parity \\
          \midrule
          Max Iter. & 100,000 & 30,000 & 30,000 & 100,000 & 100,000 & 100,000 & 100,000 & 100,000 & 35,000 & 35,000 \\
          \bottomrule
        \end{tabular}
      \end{sc}
    \end{small}
  \end{center}
  \vskip -0.1in
\end{table*}

We use cross-entropy loss for classification tasks and mean-squared error loss for regression tasks.

Regarding the hyperparameters specific to the type of recurrent cells used, no optimization was performed. For the LRU, the initialization of eigenvalues follows the procedure described by \cite{orvieto2023resurrectingrecurrentneuralnetworks}, with $[r_{\min}; r_{\max}] = [0.9; 0.999]$ (which seems to be a good compromise between hyperparameter optimization and an initialization favorable to learning long-range dependencies) and a phase in $[0; 2\pi]$.

For the CMRU and $\alpha$CMRU cells, we remove the particular initialization used in \cite{degeeter2026parallelizablememoryrecurrentunits} which aimed to bias the cell toward a retain mode ($z_t=0$) at the start of training. This trick was necessary to tackle the gradient blocking phenomenon when $z_t = 1$, but is no longer necessary thanks to the introduction of $\varepsilon$ (see \cref{subsection:vanishing}). We do not rule out that a particular initialization could accelerate the convergence of results; we acknowledge that a statistical approach could provide insight into optimal initialization values. The surrogate used in our implementation follows the original implementation $\tfrac{dH(x)}{dx}\stackrel{\text{surr}}{\approx} \tfrac{1}{1 + (\pi \alpha_{\text{surr}} x)^2}$, and we use $\alpha_{\text{surr}} = 1$.

\section{Benchmark Task Descriptions}\label{app:tasks}

This appendix provides detailed descriptions of all benchmark tasks used throughout this work. For each task, we specify the input-output structure, evaluation metric, and dataset characteristics. All tasks are evaluated using the protocol described in \cref{section:protocol}.

\subsection{Sequential Image Classification Tasks}

\subsubsection{Sequential MNIST (sMNIST)}

Sequential MNIST presents handwritten digit images (28$\times$28 pixels) as sequences of 784 grayscale pixel values in raster scan order (left-to-right, top-to-bottom). The task requires classifying the digit (0--9) based solely on the sequential pixel stream.

\paragraph{Task structure.} Each input sequence $\{x_t\}_{t=1}^{784}$ consists of normalized pixel intensities $x_t \in [0,1]$. The model processes the sequence and produces a classification decision at the final timestep. Ground truth labels $y \in \{0,1,\ldots,9\}$ provide supervision only at sequence end. We use last pooling for this task.

\paragraph{Dataset.} We use the standard MNIST training set (60,000 samples) and test set (10,000 samples) \cite{lecun1998gradient}. A validation set is created by withholding 10,000 samples from the training data.

\paragraph{Evaluation metric.} Classification accuracy on the test set.

\subsubsection{Permuted MNIST (pMNIST)}

Permuted MNIST applies a fixed random permutation to the pixel ordering of MNIST images, destroying spatial correlations while preserving information content. Conventionally regarded as more challenging than sMNIST, we challenge this assumption through evaluations presented in \cref{subsection:pmnist_results}.

\paragraph{Task structure.} Given a fixed permutation $\pi: \{1,\ldots,784\} \to \{1,\ldots,784\}$, each image is presented as the sequence $\{x_{\pi(t)}\}_{t=1}^{784}$. The permutation is identical across all samples and remains fixed during training and evaluation.

\paragraph{Variants.} We evaluate three different random permutations: (1) pMNIST; (2) pMNIST$_{\text{24}}$; and (3) pMNIST$_{\text{36}}$.

\paragraph{Dataset and evaluation.} Identical to sMNIST.

\subsubsection{Sequential CIFAR-10 (sCIFAR10)}

Sequential CIFAR-10 extends sequential image classification to natural color images with 10 object categories. Following \cite{orvieto2023resurrectingrecurrentneuralnetworks}, we use the colored version of the task and mean pooling. This task is one of those proposed by \cite{tay2020longrangearenabenchmark} to challenge Transformer architectures.

\paragraph{Task structure.} Color images (32$\times$32$\times$3) are flattened into sequences of length 1,024 by concatenating pixels in raster scan order. Each timestep presents a single pixel value $x_t \in [0,1]^3$ after normalization.

\paragraph{Dataset.} CIFAR-10 training set (50,000 samples) and test set (10,000 samples), with 10\% of training data reserved for validation \cite{krizhevsky2009learning}.

\paragraph{Evaluation metric.} Classification accuracy on the test set.

\subsection{Long-Term Memory Tasks}

All long-term memory tasks use last pooling, where supervision is provided only at the final timestep, jointly assessing memory retention and gradient propagation over the entire sequence length.

\subsubsection{Copy First Input (discrete, 15 items)}

This classification task isolates the ability to encode and retain discrete symbolic information over arbitrary sequence lengths.

\paragraph{Task structure.} At timestep $t=0$, the model observes a one-hot encoded input from a vocabulary of 15 classes. All subsequent timesteps ($t > 0$) present zero vectors. At the final timestep $t=L$, the model must classify the original input.

\paragraph{Sequence lengths.} We evaluate $L \in \{100, 300, 500, 1000, 2000, 5000, 10000\}$.

\paragraph{Dataset.} Generated synthetically with 10,000 training samples, 2,000 validation samples, and 2,000 test samples per sequence length. Class labels are uniformly distributed.

\paragraph{Evaluation metric.} Classification accuracy on the test set. Random-guess baseline is 6.67\%.

\subsubsection{Copy First Input (continuous, no noise)}

This regression task requires memorizing and reproducing a continuous scalar value with no confounding factors.

\paragraph{Task structure.} At timestep $t=0$, the model observes a scalar value $x_0 \sim \mathcal{U}(-1, 1)$. All subsequent timesteps present $x_t = 0$. The model must reproduce $x_0$ at the final timestep $t=L$.

\paragraph{Sequence lengths and dataset.} Identical to the discrete variant.

\paragraph{Evaluation metric.} Mean absolute error (MAE) on the test set. For quantized architectures, we compare against the theoretical quantization limit $\mathcal{E}^*_{\text{MAE}}$ derived in \cref{app:quantized}.

\subsubsection{Copy First Input (continuous, noisy)}

This regression task extends the continuous variant by requiring robust memory retention in the presence of continuous noise.

\paragraph{Task structure.} At timestep $t=0$, the model observes $x_0 \sim \mathcal{U}(-1, 1)$. For all $t > 0$, the model observes uniformly random noise $x_t \sim \mathcal{U}(-1, 1)$. The model must reproduce $x_0$ at timestep $t=L$ despite continuous distractor inputs.

\paragraph{Sequence lengths and dataset.} Identical to previous variants.

\paragraph{Evaluation metric.} Mean absolute error (MAE) on the test set.

\subsubsection{Copy First Input at Large Scale (continuous, noisy)}\label{app:task_copyfirst_large}

This experiment extends the copy-first-input (continuous, noisy) variant to a larger
network to verify that the persistent memory advantage is not limited to the minimal-capacity
regime.

\paragraph{Task structure.} Identical to the copy-first-input (continuous, noisy) task:
a scalar $x_0 \sim \mathcal{U}(-1, 1)$ is presented at $t=0$, followed by uniformly
random noise $x_t \sim \mathcal{U}(-1, 1)$ for all $t > 0$. The model must reproduce
$x_0$ at the final timestep.

\paragraph{Network configuration.} We evaluate $r=6$ recurrent layers with $d=256$. This is substantially larger
than the minimal-capacity experiments in the main text.

\paragraph{Sequence length.} $L=5000$.

\paragraph{Evaluation metric.} Mean absolute error (MAE) on the test set.

\subsection{Natural Language Processing Tasks}

\subsubsection{IMDb Sentiment Classification}

IMDb sentiment classification requires determining whether a movie review expresses positive or negative sentiment. Following the Long Range Arena benchmark \cite{tay2020longrangearenabenchmark}, we use byte-level encoding and mean pooling.

\paragraph{Task structure.} Movie reviews are encoded at the byte level, with each byte mapped to a learned embedding. Reviews vary in length with maximum sequence length capped at 4,000. The model processes the byte sequence with mean pooling to produce a binary classification (positive/negative).

\paragraph{Dataset.} IMDb Large Movie Review Dataset \cite{maas2011imdb} containing 50,000 reviews. We used 40,000 for training, 10,000 for validation, and the rest as the test dataset.

\paragraph{Evaluation metric.} Classification accuracy on the test set. Random-guess baseline is 50\%.

\subsubsection{ListOps}

ListOps evaluates compositional reasoning over hierarchically structured sequences of operations \cite{tay2020longrangearenabenchmark}.

\paragraph{Task structure.} Sequences represent nested operations (MAX, MIN, MEDIAN, SUM\_MOD) applied to integer arguments. Sequences are linearized in prefix notation with special tokens marking brackets and operations. Each token is mapped to a learned embedding, and the model processes sequences with mean pooling to produce the final numerical result.

\paragraph{Dataset.} We use the standard ListOps dataset \cite{nangia2018listopsdiagnosticdatasetlatent} with sequences up to length 2,000. The training set consists of 96,000 samples, while both the validation and test sets contain 2,000 samples each.

\paragraph{Evaluation metric.} Classification accuracy on the discrete output values (10 possible output classes).

\subsubsection{Pathfinder}

Pathfinder is a visual long-range dependency task from the Long Range Arena benchmark \cite{tay2020longrangearenabenchmark}. The task requires determining whether a dashed path connects two circles in a $32 \times 32$ grayscale image.

\paragraph{Task structure.} Images are flattened into sequences of length $L = 1{,}024$ pixels. The model processes the pixel sequence and produces a binary classification (connected / not connected) using mean pooling.

\paragraph{Dataset.} We use the standard LRA Pathfinder split with 160,000 training samples, 20,000 validation samples, and 20,000 test samples.

\paragraph{Evaluation metric.} Classification accuracy on the test set. Random-guess baseline is 50\%.

\subsubsection{Character-Level Language Modeling (Shakespeare)}

Character-level language modeling on the Shakespeare corpus requires predicting the next character given all preceding characters.

\paragraph{Task structure.} The corpus is encoded at the character level. At each timestep, the model observes a character token and must predict the next character. We use the full sequence without pooling, computing the loss at every timestep.

\paragraph{Dataset.} The complete works of Shakespeare, split into training, validation, and test sets by contiguous segments.

\paragraph{Evaluation metric.} Cross-entropy loss on the test set.

\subsection{Audio Classification Tasks}

All audio tasks use the Google Speech Commands dataset \cite{warden2018speechcommandsdatasetlimitedvocabulary}, which contains one-second audio clips of 35 spoken words sampled at 16 kHz. Audio is preprocessed using Mel-frequency cepstral coefficients (MFCCs), producing sequences of 101 timesteps with 13 features per timestep. We rely on \texttt{librosa} for feature extraction \cite{mcfee_2025_15006942}. For all subtasks, we construct balanced datasets with 80\% of samples used for training, 10\% for validation, and 10\% for testing. We use mean pooling.

\subsubsection{Keyword Spotting: Digits (KWS Digits)}

\paragraph{Task structure.} 11-way classification distinguishing spoken digits (``zero'' through ``nine'') and background noise.

\paragraph{Evaluation metric.} Classification accuracy on the test set. Random-guess baseline is 9.09\%.

\subsubsection{Keyword Spotting: Yes versus Others (KWS Yes versus Others)}

\paragraph{Task structure.} Binary classification detecting the word ``yes'' against a set of negative words: ``background noise'', ``no'', ``up'', ``down'', ``left'', ``right''. This represents a wake-word detection scenario.

\paragraph{Evaluation metric.} Classification accuracy on the test set. Random-guess baseline is 50\%.

\subsubsection{Keyword Spotting: All Classes (KWS All)}

\paragraph{Task structure.} 35-way classification across the full Speech Commands vocabulary.

\paragraph{Evaluation metric.} Classification accuracy on the test set. Random-guess baseline is 2.86\%.

\subsection{Modular Arithmetic Tasks}

\subsubsection{Parity}

The parity task requires computing the XOR of all binary inputs in a sequence, equivalent to determining whether the total count of ones is even or odd.

\paragraph{Task structure.} Input sequences $\{x_t\}_{t=1}^{L}$ consist of binary values $x_t \in \{0, 1\}$. The target output $y \in \{0, 1\}$ is computed as $y = \left(\sum_{t=1}^{L} x_t\right) \bmod 2$. We use last pooling, with supervision provided only at the final timestep.

\paragraph{Dataset.} Generated synthetically on-the-fly with binary inputs drawn uniformly at random. Models are trained and validated on sequences of length $L \in [50, 400]$, then evaluated on sequences of length $L \in [50, 1000]$ to assess length generalization.

\paragraph{Evaluation metric.} Classification accuracy on the test set. Random-guess baseline is 50\%. Length generalization is critical for considering the task genuinely solved.

\section{Additional Experimental Results} \label{appendix:additional_results}

This appendix presents supplementary experimental results that extend and support the findings in the main text.

\begin{figure*}[t!]
  \vskip 0in
  \begin{center}
    \includegraphics[width=\textwidth]{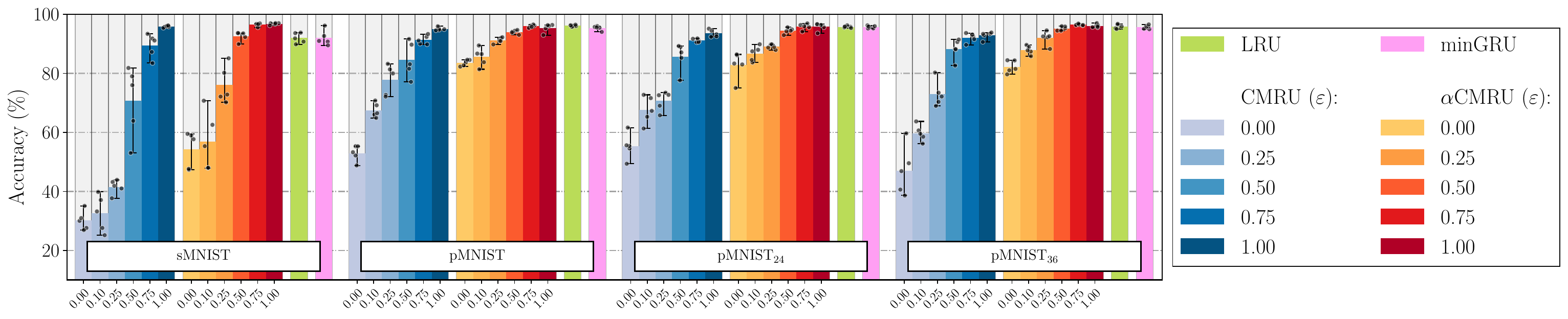}
    \caption{
      Accuracy on the pixel-by-pixel MNIST classification task across different pixel orderings.
      Each architecture is evaluated over five random initializations (shown as scattered points), with results reported as mean across seeds with min--max error bars.
      Cell types are presented from left to right within each task as: CMRU with $\varepsilon \in \{0.00, 0.10, 0.25, 0.50, 0.75, 1.00\}$, then $\alpha$CMRU for the same $\varepsilon$ values, followed by LRU and minGRU.
      Note that $\varepsilon = 0$ corresponds to the original BMRU.
      The four panels show results for: sMNIST (sequential, raster scan order), pMNIST (permutation seed 42), pMNIST$_{24}$ (permutation seed 24), and pMNIST$_{36}$ (permutation seed 36).
      Sequential MNIST proves consistently harder across all architectures, with gradient-robust models ($\varepsilon = 1$) demonstrating the largest performance advantage.
      All models use state dimension $d=32$ and a single recurrent layer. Random-guess accuracy is 10\%.
    }
    \label{figure:pmnist_challenging}
  \end{center}
\end{figure*}

\subsection{Permuted MNIST: Reconsidering the Difficulty Assumption} \label{subsection:pmnist_results}

Permuted MNIST (pMNIST) has been conventionally regarded as a more challenging benchmark than Sequential MNIST (sMNIST). This view is rooted in the observation that the fixed random permutation of pixels destroys spatial correlations, forcing models to discover hidden relationships between pixels rather than exploiting natural image structure~\cite{le2015simplewayinitializerecurrent, hu2019overcomingvanishinggradientproblem}. While this reasoning captures an important aspect of the task complexity, we argue that it presents an overly simplistic view of difficulty in sequence modeling benchmarks when dealing with recurrent architectures.

The conventional perspective conflates two distinct notions of difficulty: the inherent complexity of the solution and the difficulty of discovering that solution during training. A problem may be deemed ``harder'' not because it requires a more sophisticated solution, but because gradient-based optimization struggles to find any viable solution at all. In the context of recurrent architectures susceptible to vanishing gradients, this distinction becomes critical.

Sequential MNIST presents a specific challenge that exacerbates vanishing gradient problems: MNIST digits typically contain substantial whitespace and are centered, resulting in long sequences of near-zero (black) pixels, particularly toward the end of each sequence when presented in raster scan order. These extended uninformative regions create pathways through which gradients must propagate during backpropagation through time. In contrast, permuted MNIST distributes pixels more uniformly throughout the sequence, potentially providing more frequent gradient signals and mitigating the vanishing gradient problem.

To test this hypothesis, we evaluated our models on sMNIST and three different random permutations (pMNIST, pMNIST$_{24}$, and pMNIST$_{36}$, corresponding to random seeds 42, 24, and 36 respectively). The results, presented in \cref{figure:pmnist_challenging}, yield two key observations:

\paragraph{Sequential MNIST is harder than permuted variants.} Across all cell types tested, the sequential version consistently achieves lower accuracy than any of the three permuted versions. Our results suggest that the natural raster scan ordering is more challenging than a typical random permutation, supporting our hypothesis that the distribution of informative content throughout the sequence significantly impacts learnability for gradient-based methods.

\paragraph{Architectures robust to vanishing gradients excel on the harder task.} Both the $\alpha$CMRU and the CMRU with $\varepsilon = 1$ achieve performance equivalent to or better than the LRU and minGRU on the permuted variants. Critically, these gradient-robust architectures demonstrate substantially superior performance on the harder sMNIST task. Furthermore, increasing $\varepsilon$ toward 1 consistently improves performance while dramatically reducing initialization-induced variability across all tasks, in agreement with our theoretical gradient analysis.

\begin{figure*}[t!]
  \vskip 0in
  \begin{center}
    \includegraphics[width=\textwidth]{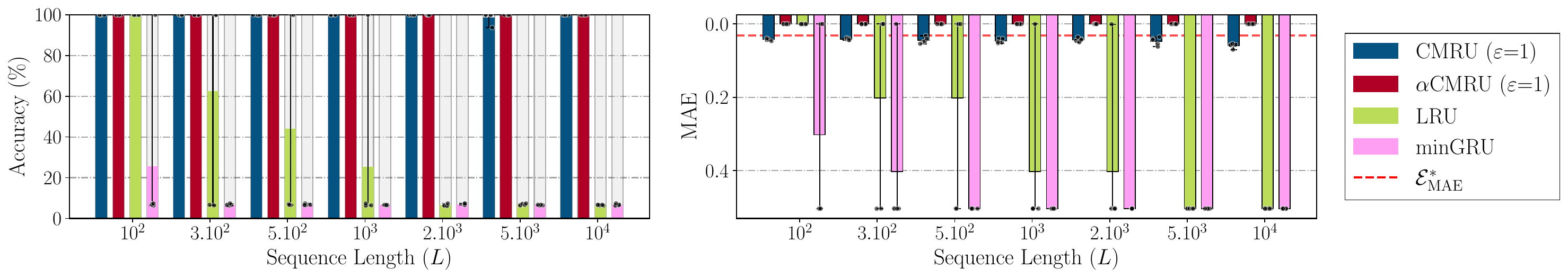}
    \caption{
      Performance on copy first input tasks as a function of sequence length $L$.
      Each architecture is evaluated over five random initializations (shown as scattered points), with results reported as mean across seeds with min--max error bars.
      Cell types are presented from left to right within each sequence length as: CMRU, $\alpha$CMRU (both with $\varepsilon = 1$), LRU, and minGRU.
      All models use hidden state dimension $d=4$ and a single recurrent layer.
      \textbf{Left:} Discrete vocabulary task (15 classes), accuracy for $L \in \{100, 300, 500, 1000, 2000, 5000, 10000\}$. Random-guess accuracy is $1/15 \approx 6.7\%$.
      \textbf{Right:} Continuous regression task (clean, noise-free). MAE with y-axis inverted (higher bars = better). The dashed red line $\mathcal{E}^*_{\text{MAE}}$ indicates the quantization limit for the CMRU at $d=4$; this bound does not apply to LRU, minGRU, or $\alpha$CMRU.
    }
    \label{figure:copy_first_combined}
  \end{center}
\end{figure*}

\subsection{Copy First Input: Discrete and Clean Continuous Variants}

The copy first input task provides a direct assessment of a model's ability to maintain persistent memory over extended sequences. In this task, a single informative signal appears at $t=0$, followed by uninformative inputs (zeros), and the model must reproduce the initial signal at the final timestep after a delay of length $L$.

We evaluate two complementary variants of this task, both presented in \cref{figure:copy_first_combined}:

\paragraph{Copy first input (discrete, 15 items).} This classification task requires the model to memorize a discrete symbol from a vocabulary of 15 classes. With sequence lengths $L \in \{100, 300, 500, 1000, 2000, 5000, 10000\}$, this task evaluates the model's ability to implement robust multi-stable memory states. For the CMRU with state dimension $d=4$, this vocabulary size represents critical capacity.

\paragraph{Copy first input (continuous, clean).} This regression task requires the model to memorize a continuous scalar value $x_0 \sim \mathcal{U}(-1, 1)$ presented at the first timestep and reproduce it at the final timestep. This variant contains only zeros after $t=0$, eliminating confounding effects from noise filtering.

The results reveal stark differences in how architectures handle these memory-intensive tasks. Both the CMRU and $\alpha$CMRU with $\varepsilon = 1$ demonstrate robust performance that remains stable as sequence length increases. In contrast, the LRU and minGRU exhibit substantial performance degradation as sequences lengthen. We argue that this degradation reflects an optimization challenge rather than an inherent architectural limitation: individual seeds that converge successfully achieve near-perfect performance, demonstrating sufficient representational capacity. The difficulty lies in reliably discovering these solutions during training across long sequences where gradients must propagate through the entire sequence length.

\subsection{Copy First Input at Large Scale}\label{app:copyfirst_large_scale}

We evaluate all architectures on the copy-first-input (continuous, noisy) task at a substantially larger scale than the main text experiments: $r=6$ recurrent layers and $d=256$, with sequence length $L=5000$. The task setup is described in \cref{app:task_copyfirst_large}.

CMRU and $\alpha$CMRU solve the task robustly across all five seeds. LRU and minGRU fail \textbf{entirely} across all seeds, yielding MAE near the trivial baseline. This result rules out the possibility that the persistent memory advantage observed at small model sizes is a consequence of insufficient representational capacity in the fading-memory baselines: the architectural distinction between persistent and fading memory determines success or failure on this task.

\subsection{Extended Audio Classification Results}

We present additional audio classification results not included in the main text: performance on the wake-word detection task (KWS Yes versus Others) and results extended to two-layer configurations ($r=2$).

\begin{table}[htbp!]
\caption{
  Accuracy (\%) for KWS Yes versus Others with $r=1$ recurrent layer.
  Results shown as mean with min--max range over multiple initializations.
}
\label{tab:audio_yes_others_r1}
\centering
  \begin{small}
    \begin{sc}
      \setlength{\tabcolsep}{3pt}
      \begin{tabular}{lcccc}
        \toprule
        $(r, d)$ & $\alpha$CMRU & CMRU & LRU & minGRU \\
        \midrule
        (1, 4) & \shortstack{98.12 \\ {\scriptsize [97.34; 98.56]}} & \shortstack{98.12 \\ {\scriptsize [97.59; 98.75]}} & \shortstack{98.18 \\ {\scriptsize [97.47; 98.59]}} & \shortstack{97.89 \\ {\scriptsize [97.03; 98.66]}} \\
        (1, 8) & \shortstack{97.96 \\ {\scriptsize [97.50; 98.19]}} & \shortstack{98.25 \\ {\scriptsize [97.84; 98.66]}} & \shortstack{98.41 \\ {\scriptsize [97.81; 99.00]}} & \shortstack{98.34 \\ {\scriptsize [97.88; 99.12]}} \\
        (1, 16) & \shortstack{97.61 \\ {\scriptsize [96.97; 98.03]}} & \shortstack{98.30 \\ {\scriptsize [97.88; 98.62]}} & \shortstack{98.45 \\ {\scriptsize [98.22; 99.06]}} & \shortstack{98.26 \\ {\scriptsize [97.56; 99.00]}} \\
        \bottomrule
      \end{tabular}
    \end{sc}
  \end{small}
  \vskip -0.1in
\end{table}

On KWS Yes versus Others, all architectures achieve approximately 98\% accuracy even at $d=4$, with differences falling within cross-seed variability. This confirms that single-layer CMRU architectures are sufficient for wake-word detection, achieving 98.30\% with $d=16$ while maintaining compatibility with 100~nW analog implementations~\cite{Fyon2026AnalogRecurrent}.

\begin{table}[htbp!]
\caption{
  Accuracy (\%) for audio classification tasks (KWS Yes versus Others, KWS Digits, KWS All) with $r=2$ recurrent layers.
  Results shown as mean with min--max range over multiple initializations.
}
\label{tab:audio_tasks_r2}
  \begin{center}
    \begin{small}
      \begin{sc}
        \setlength{\tabcolsep}{3pt}
        \begin{tabular}{lcccc}
          \toprule
          $(r, d)$ & $\alpha$CMRU & CMRU & LRU & minGRU \\
          \midrule
          \multicolumn{5}{c}{\textbf{KWS Yes versus Others}} \\
          \midrule
          (2, 4) & \shortstack{97.78 \\ {\scriptsize [96.53; 98.34]}} & \shortstack{98.45 \\ {\scriptsize [97.53; 98.91]}} & \shortstack{98.37 \\ {\scriptsize [97.62; 99.00]}} & \shortstack{98.16 \\ {\scriptsize [97.31; 98.72]}} \\
          (2, 8) & \shortstack{98.28 \\ {\scriptsize [97.88; 98.66]}} & \shortstack{98.38 \\ {\scriptsize [97.84; 98.97]}} & \shortstack{98.27 \\ {\scriptsize [97.47; 98.81]}} & \shortstack{98.49 \\ {\scriptsize [97.66; 98.91]}} \\
          (2, 16) & \shortstack{97.67 \\ {\scriptsize [96.81; 98.44]}} & \shortstack{98.35 \\ {\scriptsize [97.72; 98.75]}} & \shortstack{98.65 \\ {\scriptsize [98.19; 99.06]}} & \shortstack{98.40 \\ {\scriptsize [97.66; 99.03]}} \\
          \midrule
          \multicolumn{5}{c}{\textbf{KWS Digits}} \\
          \midrule
          (2, 4) & \shortstack{95.44 \\ {\scriptsize [94.81; 95.97]}} & \shortstack{95.43 \\ {\scriptsize [95.16; 95.72]}} & \shortstack{95.94 \\ {\scriptsize [95.00; 96.91]}} & \shortstack{96.06 \\ {\scriptsize [95.53; 96.44]}} \\
          (2, 8) & \shortstack{95.94 \\ {\scriptsize [95.69; 96.31]}} & \shortstack{95.34 \\ {\scriptsize [94.91; 95.69]}} & \shortstack{96.12 \\ {\scriptsize [95.34; 96.53]}} & \shortstack{96.19 \\ {\scriptsize [95.88; 96.53]}} \\
          (2, 16) & \shortstack{95.79 \\ {\scriptsize [94.97; 96.91]}} & \shortstack{95.82 \\ {\scriptsize [95.44; 96.16]}} & \shortstack{96.43 \\ {\scriptsize [95.91; 96.84]}} & \shortstack{96.34 \\ {\scriptsize [96.03; 96.59]}} \\
          \midrule
          \multicolumn{5}{c}{\textbf{KWS All}} \\
          \midrule
          (2, 4) & \shortstack{89.97 \\ {\scriptsize [88.75; 90.94]}} & \shortstack{89.38 \\ {\scriptsize [87.91; 90.50]}} & \shortstack{90.38 \\ {\scriptsize [87.31; 91.94]}} & \shortstack{91.38 \\ {\scriptsize [89.97; 91.81]}} \\
          (2, 8) & \shortstack{89.76 \\ {\scriptsize [86.47; 91.34]}} & \shortstack{90.09 \\ {\scriptsize [89.28; 90.53]}} & \shortstack{91.33 \\ {\scriptsize [89.66; 92.56]}} & \shortstack{91.65 \\ {\scriptsize [90.41; 92.31]}} \\
          (2, 16) & \shortstack{90.01 \\ {\scriptsize [88.09; 90.75]}} & \shortstack{90.14 \\ {\scriptsize [89.00; 90.84]}} & \shortstack{92.03 \\ {\scriptsize [91.59; 92.53]}} & \shortstack{91.86 \\ {\scriptsize [90.84; 92.59]}} \\
          \bottomrule
        \end{tabular}
      \end{sc}
    \end{small}
  \end{center}
  \vskip -0.1in
\end{table}

\paragraph{Comparison with single-layer results.} Comparing \cref{tab:audio_tasks_r2} with \cref{tab:audio_tasks_r1} reveals modest improvements from adding a second recurrent layer. On KWS Yes versus Others, all architectures achieve approximately 98\% accuracy with both $r=1$ and $r=2$, indicating that a single layer suffices for this binary detection task. On KWS Digits and KWS All, the additional layer provides marginal gains (typically 0.3--0.5\%).

\paragraph{Implications for deployment.} These results support the use of single-layer CMRU architectures for ultra-low power keyword spotting applications. The marginal accuracy gain from $r=2$ (typically $<0.5\%$) does not justify the increased circuit complexity and power consumption for analog implementations. A single-layer CMRU with $d=16$ achieves 98.30\% on wake-word detection, 95.38\% on digit recognition, and 89.76\% on full vocabulary classification.

\subsection{Extended Parity Task Analysis}

\cref{tab:Parity_alphaCMRU} presents complete results for the Parity task across all tested $\varepsilon$ values for both CMRU and $\alpha$CMRU variants. The results confirm that reflection dynamics ($\varepsilon = -1$) are necessary and sufficient for solving this task, while all other configurations fail to exceed chance performance. This finding directly supports the theoretical analysis connecting negative eigenvalues to state-tracking capabilities discussed in \cref{subsec:parity}.

\begin{table*}[t!]
\caption{
  Performance on the Parity task across different model types and $\varepsilon$ values.
  Accuracy (\%) on sequences of length $L \in [50, 1000]$ after training on $L \in [50, 400]$.
  Results shown as mean across five seeds with min--max range.
  \textbf{Top:} CMRU with $\varepsilon \in \{-1.00, 0.00, 0.10, 0.25, 0.50, 0.75, 1.00\}$ and minGRU. Note $\varepsilon = 0$ is the original BMRU.
  \textbf{Bottom:} $\alpha$CMRU with the same $\varepsilon$ values and LRU.
  Both CMRU and $\alpha$CMRU achieve 100\% with $\varepsilon = -1.00$; all other configurations perform at chance level (50\%).
}
\label{tab:Parity_alphaCMRU}
  \begin{center}
    \begin{small}
      \begin{sc}
        \setlength{\tabcolsep}{3pt}
        \begin{tabular}{lcccccccc}
          \toprule
          Model & \textbf{CMRU} & CMRU & CMRU & CMRU & CMRU & CMRU & CMRU & minGRU \\
          $\varepsilon$ & \textbf{-1.00} & 0.00 & 0.10 & 0.25 & 0.50 & 0.75 & 1.00 & --- \\
          \midrule
          Acc. (\%) & \textbf{\shortstack{100.00 \\ {\scriptsize [100.00; 100.00]}}} & \shortstack{50.01 \\ {\scriptsize [49.97; 50.06]}} & \shortstack{49.97 \\ {\scriptsize [49.34; 50.53]}} & \shortstack{49.89 \\ {\scriptsize [49.62; 50.09]}} & \shortstack{49.81 \\ {\scriptsize [49.53; 49.97]}} & \shortstack{49.95 \\ {\scriptsize [49.50; 50.47]}} & \shortstack{50.03 \\ {\scriptsize [49.47; 50.53]}} & \shortstack{49.95 \\ {\scriptsize [49.78; 50.03]}} \\
          \bottomrule
          \toprule
          Model & \textbf{$\boldsymbol{\alpha}$CMRU} & $\alpha$CMRU & $\alpha$CMRU & $\alpha$CMRU & $\alpha$CMRU & $\alpha$CMRU & $\alpha$CMRU & LRU \\
          $\varepsilon$ & \textbf{-1.00} & 0.00 & 0.10 & 0.25 & 0.50 & 0.75 & 1.00 & --- \\
          \midrule
          Acc. (\%) & \textbf{\shortstack{100.00 \\ {\scriptsize [100.00; 100.00]}}} & \shortstack{49.82 \\ {\scriptsize [49.47; 50.03]}} & \shortstack{49.81 \\ {\scriptsize [49.31; 50.06]}} & \shortstack{50.10 \\ {\scriptsize [49.38; 50.81]}} & \shortstack{49.76 \\ {\scriptsize [49.12; 50.06]}} & \shortstack{49.81 \\ {\scriptsize [49.31; 50.00]}} & \shortstack{49.82 \\ {\scriptsize [49.28; 50.34]}} & \shortstack{49.88 \\ {\scriptsize [49.31; 50.53]}} \\
          \bottomrule
        \end{tabular}
      \end{sc}
    \end{small}
  \end{center}
  \vskip -0.1in
\end{table*}

\clearpage

\section{Tabulated Results for All Figures}\label{appendix:tabulated_result}

This appendix reports results already presented elsewhere in the article in the form of figures. For the sake of scientific transparency, we provide below a summary of the quantitative results in table form, where each metric is reported by its average and the min--max interval observed over the five seeds tested. The raw data and codes are publicly available and contain more information, such as the evolution of metrics during training and the results on the validation datasets (those in the paper are on the held-out test dataset).

\begin{table}[htbp!]
  \caption{Accuracy (\%) for sMNIST with $d=32$ across different epsilon values. Note that $\varepsilon = 0$ corresponds to the original BMRU. \cref{figure:accuracy_smnist_combined}, left panel.}
  \label{tab:sMNIST_h32_epsilon}
  \begin{center}
    \begin{small}
      \begin{sc}
        \begin{tabular}{llc}
          \toprule
          Model Type & Epsilon & Accuracy (\%) \\
          \midrule
          CMRU & 0.00 & \shortstack{30.13 \\ {\scriptsize [26.94; 35.09]}} \\
          CMRU & 0.10 & \shortstack{32.59 \\ {\scriptsize [25.19; 39.84]}} \\
          CMRU & 0.25 & \shortstack{41.54 \\ {\scriptsize [37.69; 43.91]}} \\
          CMRU & 0.50 & \shortstack{70.75 \\ {\scriptsize [53.03; 81.81]}} \\
          CMRU & 0.75 & \shortstack{89.40 \\ {\scriptsize [83.47; 93.38]}} \\
          CMRU & 1.00 & \shortstack{95.92 \\ {\scriptsize [95.41; 96.19]}} \\
          $\alpha$CMRU & 0.00 & \shortstack{54.22 \\ {\scriptsize [47.31; 59.50]}} \\
          $\alpha$CMRU & 0.10 & \shortstack{56.92 \\ {\scriptsize [47.94; 70.72]}} \\
          $\alpha$CMRU & 0.25 & \shortstack{76.08 \\ {\scriptsize [70.19; 85.09]}} \\
          $\alpha$CMRU & 0.50 & \shortstack{92.59 \\ {\scriptsize [89.91; 93.66]}} \\
          $\alpha$CMRU & 0.75 & \shortstack{96.45 \\ {\scriptsize [95.56; 96.91]}} \\
          $\alpha$CMRU & 1.00 & \shortstack{96.76 \\ {\scriptsize [96.22; 97.00]}} \\
          \bottomrule
        \end{tabular}
      \end{sc}
    \end{small}
  \end{center}
  \vskip -0.1in
\end{table}

\begin{table}[htbp!]
  \caption{Accuracy (\%) for sMNIST with $\varepsilon=1.0$ across different state dimensions. \cref{figure:accuracy_smnist_combined}, right panel.}
  \label{tab:sMNIST_eps1_state_dim}
  \begin{center}
    \begin{small}
      \begin{sc}
        \setlength{\tabcolsep}{3pt}
        \begin{tabular}{lllc}
          \toprule
          Model Type & State Dim & Epsilon & Accuracy (\%) \\
          \midrule
          $\alpha$CMRU & 4 & 1.00 & \shortstack{88.34 \\ {\scriptsize [87.38; 89.53]}} \\
          $\alpha$CMRU & 8 & 1.00 & \shortstack{93.49 \\ {\scriptsize [91.56; 94.62]}} \\
          $\alpha$CMRU & 16 & 1.00 & \shortstack{95.43 \\ {\scriptsize [93.03; 96.47]}} \\
          $\alpha$CMRU & 32 & 1.00 & \shortstack{96.76 \\ {\scriptsize [96.22; 97.00]}} \\
          $\alpha$CMRU & 64 & 1.00 & \shortstack{96.88 \\ {\scriptsize [96.66; 97.12]}} \\
          $\alpha$CMRU & 128 & 1.00 & \shortstack{97.16 \\ {\scriptsize [96.72; 97.91]}} \\
          $\alpha$CMRU & 256 & 1.00 & \shortstack{97.25 \\ {\scriptsize [96.69; 98.06]}} \\
          CMRU & 4 & 1.00 & \shortstack{83.61 \\ {\scriptsize [80.97; 85.09]}} \\
          CMRU & 8 & 1.00 & \shortstack{91.49 \\ {\scriptsize [90.53; 92.69]}} \\
          CMRU & 16 & 1.00 & \shortstack{94.83 \\ {\scriptsize [93.12; 95.84]}} \\
          CMRU & 32 & 1.00 & \shortstack{95.92 \\ {\scriptsize [95.41; 96.19]}} \\
          CMRU & 64 & 1.00 & \shortstack{97.05 \\ {\scriptsize [96.69; 97.28]}} \\
          CMRU & 128 & 1.00 & \shortstack{96.20 \\ {\scriptsize [95.16; 97.00]}} \\
          CMRU & 256 & 1.00 & \shortstack{96.16 \\ {\scriptsize [95.81; 96.38]}} \\
          LRU & 4 & --- & \shortstack{72.88 \\ {\scriptsize [53.87; 83.22]}} \\
          LRU & 8 & --- & \shortstack{78.08 \\ {\scriptsize [62.72; 88.34]}} \\
          LRU & 16 & --- & \shortstack{88.98 \\ {\scriptsize [82.78; 91.53]}} \\
          LRU & 32 & --- & \shortstack{91.97 \\ {\scriptsize [89.84; 93.78]}} \\
          LRU & 64 & --- & \shortstack{92.59 \\ {\scriptsize [90.41; 93.84]}} \\
          LRU & 128 & --- & \shortstack{94.08 \\ {\scriptsize [91.66; 96.06]}} \\
          LRU & 256 & --- & \shortstack{95.61 \\ {\scriptsize [94.69; 96.31]}} \\
          minGRU & 4 & --- & \shortstack{75.88 \\ {\scriptsize [66.16; 86.09]}} \\
          minGRU & 8 & --- & \shortstack{81.88 \\ {\scriptsize [73.28; 88.69]}} \\
          minGRU & 16 & --- & \shortstack{92.17 \\ {\scriptsize [90.34; 93.28]}} \\
          minGRU & 32 & --- & \shortstack{92.01 \\ {\scriptsize [89.50; 96.16]}} \\
          minGRU & 64 & --- & \shortstack{93.39 \\ {\scriptsize [88.50; 96.47]}} \\
          minGRU & 128 & --- & \shortstack{93.88 \\ {\scriptsize [91.12; 96.00]}} \\
          minGRU & 256 & --- & \shortstack{96.62 \\ {\scriptsize [94.62; 97.62]}} \\
          \bottomrule
        \end{tabular}
      \end{sc}
    \end{small}
  \end{center}
  \vskip -0.1in
\end{table}

\begin{table*}[t!]
  \caption{MAE for Copy First Continuous (noisy) tasks across architectures ($d$, $r$) and sequence lengths (Part 1 of 3). \cref{figure:mae_copy_first_combined}.}
  \label{tab:copy_first_continuous_noisy_by_arch_part1}
  \begin{center}
    \begin{small}
      \begin{sc}
        \begin{tabular}{llllc}
          \toprule
          Model Type & Seq Length & State Dim & Num Recs & MAE \\
          \midrule
          $\alpha$CMRU & 100 & 4 & 1 & \shortstack{0.0002 \\ {\scriptsize [0.0002; 0.0002]}} \\
          $\alpha$CMRU & 300 & 4 & 1 & \shortstack{0.0002 \\ {\scriptsize [0.0001; 0.0002]}} \\
          $\alpha$CMRU & 500 & 4 & 1 & \shortstack{0.0002 \\ {\scriptsize [0.0002; 0.0002]}} \\
          $\alpha$CMRU & 1000 & 4 & 1 & \shortstack{0.0002 \\ {\scriptsize [0.0002; 0.0003]}} \\
          $\alpha$CMRU & 2000 & 4 & 1 & \shortstack{0.0002 \\ {\scriptsize [0.0002; 0.0002]}} \\
          $\alpha$CMRU & 5000 & 4 & 1 & \shortstack{0.0005 \\ {\scriptsize [0.0002; 0.0014]}} \\
          $\alpha$CMRU & 10000 & 4 & 1 & \shortstack{0.1009 \\ {\scriptsize [0.0002; 0.5038]}} \\
          CMRU & 100 & 4 & 1 & \shortstack{0.0401 \\ {\scriptsize [0.0309; 0.0466]}} \\
          CMRU & 300 & 4 & 1 & \shortstack{0.0404 \\ {\scriptsize [0.0342; 0.0457]}} \\
          CMRU & 500 & 4 & 1 & \shortstack{0.0383 \\ {\scriptsize [0.0362; 0.0399]}} \\
          CMRU & 1000 & 4 & 1 & \shortstack{0.0402 \\ {\scriptsize [0.0362; 0.0432]}} \\
          CMRU & 2000 & 4 & 1 & \shortstack{0.1743 \\ {\scriptsize [0.0378; 0.5037]}} \\
          CMRU & 5000 & 4 & 1 & \shortstack{0.1538 \\ {\scriptsize [0.0363; 0.5038]}} \\
          CMRU & 10000 & 4 & 1 & \shortstack{0.1357 \\ {\scriptsize [0.0352; 0.5037]}} \\
          LRU & 100 & 4 & 1 & \shortstack{0.1011 \\ {\scriptsize [0.0002; 0.5038]}} \\
          LRU & 300 & 4 & 1 & \shortstack{0.3025 \\ {\scriptsize [0.0005; 0.5038]}} \\
          LRU & 500 & 4 & 1 & \shortstack{0.4031 \\ {\scriptsize [0.0004; 0.5040]}} \\
          LRU & 1000 & 4 & 1 & \shortstack{0.4033 \\ {\scriptsize [0.0012; 0.5039]}} \\
          LRU & 2000 & 4 & 1 & \shortstack{0.5039 \\ {\scriptsize [0.5038; 0.5044]}} \\
          LRU & 5000 & 4 & 1 & \shortstack{0.5039 \\ {\scriptsize [0.5037; 0.5042]}} \\
          LRU & 10000 & 4 & 1 & \shortstack{0.5038 \\ {\scriptsize [0.5037; 0.5038]}} \\
          minGRU & 100 & 4 & 1 & \shortstack{0.4031 \\ {\scriptsize [0.0002; 0.5039]}} \\
          minGRU & 300 & 4 & 1 & \shortstack{0.5038 \\ {\scriptsize [0.5038; 0.5039]}} \\
          minGRU & 500 & 4 & 1 & \shortstack{0.5038 \\ {\scriptsize [0.5038; 0.5038]}} \\
          minGRU & 1000 & 4 & 1 & \shortstack{0.5038 \\ {\scriptsize [0.5038; 0.5039]}} \\
          minGRU & 2000 & 4 & 1 & \shortstack{0.5039 \\ {\scriptsize [0.5038; 0.5042]}} \\
          minGRU & 5000 & 4 & 1 & \shortstack{0.5038 \\ {\scriptsize [0.5037; 0.5041]}} \\
          minGRU & 10000 & 4 & 1 & \shortstack{0.5037 \\ {\scriptsize [0.5036; 0.5038]}} \\
          $\alpha$CMRU & 100 & 16 & 1 & \shortstack{0.0002 \\ {\scriptsize [0.0002; 0.0002]}} \\
          $\alpha$CMRU & 300 & 16 & 1 & \shortstack{0.0002 \\ {\scriptsize [0.0002; 0.0002]}} \\
          $\alpha$CMRU & 500 & 16 & 1 & \shortstack{0.0002 \\ {\scriptsize [0.0002; 0.0002]}} \\
          $\alpha$CMRU & 1000 & 16 & 1 & \shortstack{0.0002 \\ {\scriptsize [0.0002; 0.0002]}} \\
          \bottomrule
        \end{tabular}
      \end{sc}
    \end{small}
  \end{center}
  \vskip -0.1in
\end{table*}

\begin{table*}[t!]
  \caption{MAE for Copy First Continuous (noisy) tasks across architectures ($d$, $r$) and sequence lengths (Part 2 of 3). \cref{figure:mae_copy_first_combined}.}
  \label{tab:copy_first_continuous_noisy_by_arch_part2}
  \begin{center}
    \begin{small}
      \begin{sc}
        \begin{tabular}{llllc}
          \toprule
          Model Type & Seq Length & State Dim & Num Recs & MAE \\
          \midrule
          $\alpha$CMRU & 2000 & 16 & 1 & \shortstack{0.0002 \\ {\scriptsize [0.0001; 0.0002]}} \\
          $\alpha$CMRU & 5000 & 16 & 1 & \shortstack{0.0002 \\ {\scriptsize [0.0002; 0.0003]}} \\
          $\alpha$CMRU & 10000 & 16 & 1 & \shortstack{0.1010 \\ {\scriptsize [0.0002; 0.5037]}} \\
          CMRU & 100 & 16 & 1 & \shortstack{0.0119 \\ {\scriptsize [0.0099; 0.0143]}} \\
          CMRU & 300 & 16 & 1 & \shortstack{0.0143 \\ {\scriptsize [0.0112; 0.0163]}} \\
          CMRU & 500 & 16 & 1 & \shortstack{0.0147 \\ {\scriptsize [0.0139; 0.0155]}} \\
          CMRU & 1000 & 16 & 1 & \shortstack{0.0148 \\ {\scriptsize [0.0128; 0.0159]}} \\
          CMRU & 2000 & 16 & 1 & \shortstack{0.0154 \\ {\scriptsize [0.0122; 0.0186]}} \\
          CMRU & 5000 & 16 & 1 & \shortstack{0.1142 \\ {\scriptsize [0.0168; 0.5038]}} \\
          CMRU & 10000 & 16 & 1 & \shortstack{0.1142 \\ {\scriptsize [0.0168; 0.5038]}} \\
          LRU & 100 & 16 & 1 & \shortstack{0.0002 \\ {\scriptsize [0.0002; 0.0003]}} \\
          LRU & 300 & 16 & 1 & \shortstack{0.0004 \\ {\scriptsize [0.0002; 0.0006]}} \\
          LRU & 500 & 16 & 1 & \shortstack{0.1012 \\ {\scriptsize [0.0003; 0.5038]}} \\
          LRU & 1000 & 16 & 1 & \shortstack{0.2021 \\ {\scriptsize [0.0003; 0.5039]}} \\
          LRU & 2000 & 16 & 1 & \shortstack{0.5039 \\ {\scriptsize [0.5038; 0.5041]}} \\
          LRU & 5000 & 16 & 1 & \shortstack{0.5039 \\ {\scriptsize [0.5037; 0.5046]}} \\
          LRU & 10000 & 16 & 1 & \shortstack{0.5037 \\ {\scriptsize [0.5036; 0.5038]}} \\
          minGRU & 100 & 16 & 1 & \shortstack{0.3023 \\ {\scriptsize [0.0002; 0.5038]}} \\
          minGRU & 300 & 16 & 1 & \shortstack{0.4030 \\ {\scriptsize [0.0002; 0.5038]}} \\
          minGRU & 500 & 16 & 1 & \shortstack{0.5038 \\ {\scriptsize [0.5038; 0.5038]}} \\
          minGRU & 1000 & 16 & 1 & \shortstack{0.5038 \\ {\scriptsize [0.5037; 0.5039]}} \\
          minGRU & 2000 & 16 & 1 & \shortstack{0.5039 \\ {\scriptsize [0.5038; 0.5042]}} \\
          minGRU & 5000 & 16 & 1 & \shortstack{0.5039 \\ {\scriptsize [0.5038; 0.5043]}} \\
          minGRU & 10000 & 16 & 1 & \shortstack{0.5038 \\ {\scriptsize [0.5037; 0.5038]}} \\
          $\alpha$CMRU & 100 & 16 & 4 & \shortstack{0.0001 \\ {\scriptsize [0.0001; 0.0002]}} \\
          $\alpha$CMRU & 300 & 16 & 4 & \shortstack{0.0001 \\ {\scriptsize [0.0001; 0.0002]}} \\
          $\alpha$CMRU & 500 & 16 & 4 & \shortstack{0.0001 \\ {\scriptsize [0.0001; 0.0002]}} \\
          $\alpha$CMRU & 1000 & 16 & 4 & \shortstack{0.0001 \\ {\scriptsize [0.0001; 0.0002]}} \\
          $\alpha$CMRU & 2000 & 16 & 4 & \shortstack{0.0001 \\ {\scriptsize [0.0000; 0.0001]}} \\
          $\alpha$CMRU & 5000 & 16 & 4 & \shortstack{0.0001 \\ {\scriptsize [0.0001; 0.0002]}} \\
          $\alpha$CMRU & 10000 & 16 & 4 & \shortstack{0.0004 \\ {\scriptsize [0.0001; 0.0014]}} \\
          CMRU & 100 & 16 & 4 & \shortstack{0.0019 \\ {\scriptsize [0.0014; 0.0030]}} \\
          \bottomrule
        \end{tabular}
      \end{sc}
    \end{small}
  \end{center}
  \vskip -0.1in
\end{table*}

\begin{table*}[t!]
  \caption{MAE for Copy First Continuous (noisy) tasks across architectures ($d$, $r$) and sequence lengths (Part 3 of 3). \cref{figure:mae_copy_first_combined}.}
  \label{tab:copy_first_continuous_noisy_by_arch_part3}
  \begin{center}
    \begin{small}
      \begin{sc}
        \begin{tabular}{llllc}
          \toprule
          Model Type & Seq Length & State Dim & Num Recs & MAE \\
          \midrule
          CMRU & 300 & 16 & 4 & \shortstack{0.0036 \\ {\scriptsize [0.0021; 0.0048]}} \\
          CMRU & 500 & 16 & 4 & \shortstack{0.0034 \\ {\scriptsize [0.0020; 0.0050]}} \\
          CMRU & 1000 & 16 & 4 & \shortstack{0.0032 \\ {\scriptsize [0.0023; 0.0049]}} \\
          CMRU & 2000 & 16 & 4 & \shortstack{0.0034 \\ {\scriptsize [0.0026; 0.0040]}} \\
          CMRU & 5000 & 16 & 4 & \shortstack{0.0049 \\ {\scriptsize [0.0027; 0.0128]}} \\
          CMRU & 10000 & 16 & 4 & \shortstack{0.0066 \\ {\scriptsize [0.0033; 0.0134]}} \\
          LRU & 100 & 16 & 4 & \shortstack{0.0002 \\ {\scriptsize [0.0002; 0.0003]}} \\
          LRU & 300 & 16 & 4 & \shortstack{0.0003 \\ {\scriptsize [0.0001; 0.0005]}} \\
          LRU & 500 & 16 & 4 & \shortstack{0.0006 \\ {\scriptsize [0.0002; 0.0011]}} \\
          LRU & 1000 & 16 & 4 & \shortstack{0.1022 \\ {\scriptsize [0.0003; 0.5038]}} \\
          LRU & 2000 & 16 & 4 & \shortstack{0.4038 \\ {\scriptsize [0.0038; 0.5038]}} \\
          LRU & 5000 & 16 & 4 & \shortstack{0.5038 \\ {\scriptsize [0.5038; 0.5040]}} \\
          LRU & 10000 & 16 & 4 & \shortstack{0.5038 \\ {\scriptsize [0.5037; 0.5038]}} \\
          minGRU & 100 & 16 & 4 & \shortstack{0.1012 \\ {\scriptsize [0.0001; 0.5038]}} \\
          minGRU & 300 & 16 & 4 & \shortstack{0.1013 \\ {\scriptsize [0.0003; 0.5038]}} \\
          minGRU & 500 & 16 & 4 & \shortstack{0.2021 \\ {\scriptsize [0.0002; 0.5038]}} \\
          minGRU & 1000 & 16 & 4 & \shortstack{0.5038 \\ {\scriptsize [0.5038; 0.5038]}} \\
          minGRU & 2000 & 16 & 4 & \shortstack{0.5039 \\ {\scriptsize [0.5038; 0.5046]}} \\
          minGRU & 5000 & 16 & 4 & \shortstack{0.5039 \\ {\scriptsize [0.5038; 0.5044]}} \\
          minGRU & 10000 & 16 & 4 & \shortstack{0.5039 \\ {\scriptsize [0.5038; 0.5043]}} \\
          \bottomrule
        \end{tabular}
      \end{sc}
    \end{small}
  \end{center}
  \vskip -0.1in
\end{table*}

\begin{table}[htbp!]
  \caption{Accuracy (\%) comparison across sequential MNIST permutations for $d=32$ with different epsilon values (Part 1 of 2). Note that $\varepsilon = 0$ corresponds to the original BMRU. \cref{figure:pmnist_challenging}.}
  \label{tab:epsilon_comparison_h32_part1}
  \begin{center}
    \begin{small}
      \begin{sc}
        \setlength{\tabcolsep}{3pt}
        \begin{tabular}{lllc}
          \toprule
          Task Name & Model Type & Epsilon & Accuracy (\%) \\
          \midrule
          sMNIST & CMRU & 0.00 & \shortstack{30.13 \\ {\scriptsize [26.94; 35.09]}} \\
          sMNIST & CMRU & 0.10 & \shortstack{32.59 \\ {\scriptsize [25.19; 39.84]}} \\
          sMNIST & CMRU & 0.25 & \shortstack{41.54 \\ {\scriptsize [37.69; 43.91]}} \\
          sMNIST & CMRU & 0.50 & \shortstack{70.75 \\ {\scriptsize [53.03; 81.81]}} \\
          sMNIST & CMRU & 0.75 & \shortstack{89.40 \\ {\scriptsize [83.47; 93.38]}} \\
          sMNIST & CMRU & 1.00 & \shortstack{95.92 \\ {\scriptsize [95.41; 96.19]}} \\
          sMNIST & $\alpha$CMRU & 0.00 & \shortstack{54.22 \\ {\scriptsize [47.31; 59.50]}} \\
          sMNIST & $\alpha$CMRU & 0.10 & \shortstack{56.92 \\ {\scriptsize [47.94; 70.72]}} \\
          sMNIST & $\alpha$CMRU & 0.25 & \shortstack{76.08 \\ {\scriptsize [70.19; 85.09]}} \\
          sMNIST & $\alpha$CMRU & 0.50 & \shortstack{92.59 \\ {\scriptsize [89.91; 93.66]}} \\
          sMNIST & $\alpha$CMRU & 0.75 & \shortstack{96.45 \\ {\scriptsize [95.56; 96.91]}} \\
          sMNIST & $\alpha$CMRU & 1.00 & \shortstack{96.76 \\ {\scriptsize [96.22; 97.00]}} \\
          sMNIST & LRU & --- & \shortstack{91.97 \\ {\scriptsize [89.84; 93.78]}} \\
          sMNIST & minGRU & --- & \shortstack{92.01 \\ {\scriptsize [89.50; 96.16]}} \\
          pMNIST & CMRU & 0.00 & \shortstack{52.98 \\ {\scriptsize [48.81; 55.31]}} \\
          pMNIST & CMRU & 0.10 & \shortstack{67.49 \\ {\scriptsize [64.91; 70.78]}} \\
          pMNIST & CMRU & 0.25 & \shortstack{77.87 \\ {\scriptsize [72.12; 83.22]}} \\
          pMNIST & CMRU & 0.50 & \shortstack{84.62 \\ {\scriptsize [77.12; 91.62]}} \\
          pMNIST & CMRU & 0.75 & \shortstack{91.32 \\ {\scriptsize [89.81; 93.31]}} \\
          pMNIST & CMRU & 1.00 & \shortstack{95.06 \\ {\scriptsize [94.56; 95.94]}} \\
          pMNIST & $\alpha$CMRU & 0.00 & \shortstack{83.58 \\ {\scriptsize [82.28; 84.59]}} \\
          pMNIST & $\alpha$CMRU & 0.10 & \shortstack{85.56 \\ {\scriptsize [81.38; 89.44]}} \\
          pMNIST & $\alpha$CMRU & 0.25 & \shortstack{91.19 \\ {\scriptsize [89.78; 92.22]}} \\
          pMNIST & $\alpha$CMRU & 0.50 & \shortstack{93.99 \\ {\scriptsize [93.03; 94.69]}} \\
          pMNIST & $\alpha$CMRU & 0.75 & \shortstack{95.99 \\ {\scriptsize [95.38; 96.53]}} \\
          pMNIST & $\alpha$CMRU & 1.00 & \shortstack{95.39 \\ {\scriptsize [92.97; 96.38]}} \\
          pMNIST & LRU & --- & \shortstack{96.12 \\ {\scriptsize [95.75; 96.47]}} \\
          pMNIST & minGRU & --- & \shortstack{95.35 \\ {\scriptsize [94.03; 96.00]}} \\
          $\text{pMNIST}_{24}$ & CMRU & 0.00 & \shortstack{55.31 \\ {\scriptsize [49.38; 61.62]}} \\
          $\text{pMNIST}_{24}$ & CMRU & 0.10 & \shortstack{67.64 \\ {\scriptsize [61.31; 72.72]}} \\
          $\text{pMNIST}_{24}$ & CMRU & 0.25 & \shortstack{70.71 \\ {\scriptsize [65.78; 73.44]}} \\
          $\text{pMNIST}_{24}$ & CMRU & 0.50 & \shortstack{85.63 \\ {\scriptsize [77.62; 89.25]}} \\
          \bottomrule
        \end{tabular}
      \end{sc}
    \end{small}
  \end{center}
  \vskip -0.1in
\end{table}

\begin{table}[htbp!]
  \caption{Accuracy (\%) comparison across sequential MNIST permutations for $d=32$ with different epsilon values (Part 2 of 2). Note that $\varepsilon = 0$ corresponds to the original BMRU. \cref{figure:pmnist_challenging}.}
  \label{tab:epsilon_comparison_h32_part2}
  \begin{center}
    \begin{small}
      \begin{sc}
        \setlength{\tabcolsep}{3pt}
        \begin{tabular}{lllc}
          \toprule
          Task Name & Model Type & Epsilon & Accuracy (\%) \\
          \midrule
          $\text{pMNIST}_{24}$ & CMRU & 0.75 & \shortstack{91.09 \\ {\scriptsize [90.50; 91.81]}} \\
          $\text{pMNIST}_{24}$ & CMRU & 1.00 & \shortstack{93.52 \\ {\scriptsize [92.38; 95.12]}} \\
          $\text{pMNIST}_{24}$ & $\alpha$CMRU & 0.00 & \shortstack{82.78 \\ {\scriptsize [75.03; 86.41]}} \\
          $\text{pMNIST}_{24}$ & $\alpha$CMRU & 0.10 & \shortstack{86.70 \\ {\scriptsize [83.69; 89.84]}} \\
          $\text{pMNIST}_{24}$ & $\alpha$CMRU & 0.25 & \shortstack{89.03 \\ {\scriptsize [87.91; 89.81]}} \\
          $\text{pMNIST}_{24}$ & $\alpha$CMRU & 0.50 & \shortstack{94.42 \\ {\scriptsize [92.94; 95.59]}} \\
          $\text{pMNIST}_{24}$ & $\alpha$CMRU & 0.75 & \shortstack{95.80 \\ {\scriptsize [94.16; 96.91]}} \\
          $\text{pMNIST}_{24}$ & $\alpha$CMRU & 1.00 & \shortstack{95.77 \\ {\scriptsize [93.56; 96.69]}} \\
          $\text{pMNIST}_{24}$ & LRU & --- & \shortstack{95.74 \\ {\scriptsize [95.34; 96.25]}} \\
          $\text{pMNIST}_{24}$ & minGRU & --- & \shortstack{95.69 \\ {\scriptsize [95.19; 96.25]}} \\
          $\text{pMNIST}_{36}$ & CMRU & 0.00 & \shortstack{47.07 \\ {\scriptsize [38.62; 59.66]}} \\
          $\text{pMNIST}_{36}$ & CMRU & 0.10 & \shortstack{59.73 \\ {\scriptsize [56.19; 63.72]}} \\
          $\text{pMNIST}_{36}$ & CMRU & 0.25 & \shortstack{73.03 \\ {\scriptsize [68.97; 80.25]}} \\
          $\text{pMNIST}_{36}$ & CMRU & 0.50 & \shortstack{88.22 \\ {\scriptsize [82.62; 91.44]}} \\
          $\text{pMNIST}_{36}$ & CMRU & 0.75 & \shortstack{92.09 \\ {\scriptsize [89.66; 93.66]}} \\
          $\text{pMNIST}_{36}$ & CMRU & 1.00 & \shortstack{92.81 \\ {\scriptsize [90.62; 93.81]}} \\
          $\text{pMNIST}_{36}$ & $\alpha$CMRU & 0.00 & \shortstack{82.20 \\ {\scriptsize [79.66; 84.44]}} \\
          $\text{pMNIST}_{36}$ & $\alpha$CMRU & 0.10 & \shortstack{87.90 \\ {\scriptsize [85.81; 89.62]}} \\
          $\text{pMNIST}_{36}$ & $\alpha$CMRU & 0.25 & \shortstack{92.01 \\ {\scriptsize [88.25; 94.53]}} \\
          $\text{pMNIST}_{36}$ & $\alpha$CMRU & 0.50 & \shortstack{95.20 \\ {\scriptsize [94.47; 95.97]}} \\
          $\text{pMNIST}_{36}$ & $\alpha$CMRU & 0.75 & \shortstack{96.46 \\ {\scriptsize [96.22; 96.84]}} \\
          $\text{pMNIST}_{36}$ & $\alpha$CMRU & 1.00 & \shortstack{96.06 \\ {\scriptsize [95.56; 97.00]}} \\
          $\text{pMNIST}_{36}$ & LRU & --- & \shortstack{95.86 \\ {\scriptsize [95.03; 96.75]}} \\
          $\text{pMNIST}_{36}$ & minGRU & --- & \shortstack{95.64 \\ {\scriptsize [94.94; 96.56]}} \\
          \bottomrule
        \end{tabular}
      \end{sc}
    \end{small}
  \end{center}
  \vskip -0.1in
\end{table}

\begin{table}[htbp!]
  \caption{Accuracy (\%) for Copy First Discrete (15 items) tasks across sequence lengths. \cref{figure:copy_first_combined}, left panel.}
  \label{tab:copy_first_discrete}
  \begin{center}
    \begin{small}
      \begin{sc}
        \begin{tabular}{llc}
          \toprule
          Model Type & Seq Length & Accuracy (\%) \\
          \midrule
          CMRU & 100 & \shortstack{100.00 \\ {\scriptsize [100.00; 100.00]}} \\
          $\alpha$CMRU & 100 & \shortstack{100.00 \\ {\scriptsize [100.00; 100.00]}} \\
          LRU & 100 & \shortstack{100.00 \\ {\scriptsize [100.00; 100.00]}} \\
          minGRU & 100 & \shortstack{25.48 \\ {\scriptsize [6.28; 100.00]}} \\
          CMRU & 300 & \shortstack{100.00 \\ {\scriptsize [100.00; 100.00]}} \\
          $\alpha$CMRU & 300 & \shortstack{100.00 \\ {\scriptsize [100.00; 100.00]}} \\
          LRU & 300 & \shortstack{62.61 \\ {\scriptsize [6.41; 100.00]}} \\
          minGRU & 300 & \shortstack{6.67 \\ {\scriptsize [6.28; 7.34]}} \\
          CMRU & 500 & \shortstack{100.00 \\ {\scriptsize [100.00; 100.00]}} \\
          $\alpha$CMRU & 500 & \shortstack{100.00 \\ {\scriptsize [100.00; 100.00]}} \\
          LRU & 500 & \shortstack{44.04 \\ {\scriptsize [6.50; 100.00]}} \\
          minGRU & 500 & \shortstack{6.71 \\ {\scriptsize [6.50; 7.34]}} \\
          CMRU & 1000 & \shortstack{100.00 \\ {\scriptsize [100.00; 100.00]}} \\
          $\alpha$CMRU & 1000 & \shortstack{100.00 \\ {\scriptsize [100.00; 100.00]}} \\
          LRU & 1000 & \shortstack{25.31 \\ {\scriptsize [6.28; 100.00]}} \\
          minGRU & 1000 & \shortstack{6.49 \\ {\scriptsize [6.28; 6.66]}} \\
          CMRU & 2000 & \shortstack{100.00 \\ {\scriptsize [100.00; 100.00]}} \\
          $\alpha$CMRU & 2000 & \shortstack{100.00 \\ {\scriptsize [100.00; 100.00]}} \\
          LRU & 2000 & \shortstack{6.53 \\ {\scriptsize [6.12; 7.34]}} \\
          minGRU & 2000 & \shortstack{6.99 \\ {\scriptsize [6.50; 7.34]}} \\
          CMRU & 5000 & \shortstack{98.74 \\ {\scriptsize [93.72; 100.00]}} \\
          $\alpha$CMRU & 5000 & \shortstack{100.00 \\ {\scriptsize [100.00; 100.00]}} \\
          LRU & 5000 & \shortstack{6.81 \\ {\scriptsize [6.41; 7.34]}} \\
          minGRU & 5000 & \shortstack{6.50 \\ {\scriptsize [6.28; 6.66]}} \\
          CMRU & 10000 & \shortstack{100.00 \\ {\scriptsize [100.00; 100.00]}} \\
          $\alpha$CMRU & 10000 & \shortstack{100.00 \\ {\scriptsize [100.00; 100.00]}} \\
          LRU & 10000 & \shortstack{6.50 \\ {\scriptsize [6.28; 6.66]}} \\
          minGRU & 10000 & \shortstack{6.64 \\ {\scriptsize [6.28; 7.34]}} \\
          \bottomrule
        \end{tabular}
      \end{sc}
    \end{small}
  \end{center}
  \vskip -0.1in
\end{table}

\begin{table}[htbp!]
  \caption{MAE for Copy First Continuous (no noise) tasks with $d=4$ across sequence lengths. \cref{figure:copy_first_combined}, right panel.}
  \label{tab:copy_first_continuous}
  \begin{center}
    \begin{small}
      \begin{sc}
        \begin{tabular}{llc}
          \toprule
          Model Type & Seq Length & MAE \\
          \midrule
          CMRU & 100 & \shortstack{0.0424 \\ {\scriptsize [0.0392; 0.0465]}} \\
          $\alpha$CMRU & 100 & \shortstack{0.0002 \\ {\scriptsize [0.0002; 0.0002]}} \\
          LRU & 100 & \shortstack{0.0002 \\ {\scriptsize [0.0002; 0.0003]}} \\
          minGRU & 100 & \shortstack{0.3023 \\ {\scriptsize [0.0002; 0.5038]}} \\
          CMRU & 300 & \shortstack{0.0412 \\ {\scriptsize [0.0380; 0.0438]}} \\
          $\alpha$CMRU & 300 & \shortstack{0.0002 \\ {\scriptsize [0.0001; 0.0002]}} \\
          LRU & 300 & \shortstack{0.2017 \\ {\scriptsize [0.0001; 0.5041]}} \\
          minGRU & 300 & \shortstack{0.4031 \\ {\scriptsize [0.0002; 0.5038]}} \\
          CMRU & 500 & \shortstack{0.0443 \\ {\scriptsize [0.0322; 0.0539]}} \\
          $\alpha$CMRU & 500 & \shortstack{0.0002 \\ {\scriptsize [0.0002; 0.0002]}} \\
          LRU & 500 & \shortstack{0.2016 \\ {\scriptsize [0.0002; 0.5038]}} \\
          minGRU & 500 & \shortstack{0.5038 \\ {\scriptsize [0.5038; 0.5039]}} \\
          CMRU & 1000 & \shortstack{0.0465 \\ {\scriptsize [0.0397; 0.0525]}} \\
          $\alpha$CMRU & 1000 & \shortstack{0.0002 \\ {\scriptsize [0.0002; 0.0003]}} \\
          LRU & 1000 & \shortstack{0.4031 \\ {\scriptsize [0.0002; 0.5039]}} \\
          minGRU & 1000 & \shortstack{0.5038 \\ {\scriptsize [0.5038; 0.5039]}} \\
          CMRU & 2000 & \shortstack{0.0432 \\ {\scriptsize [0.0361; 0.0499]}} \\
          $\alpha$CMRU & 2000 & \shortstack{0.0002 \\ {\scriptsize [0.0002; 0.0002]}} \\
          LRU & 2000 & \shortstack{0.4032 \\ {\scriptsize [0.0006; 0.5038]}} \\
          minGRU & 2000 & \shortstack{0.5038 \\ {\scriptsize [0.5038; 0.5038]}} \\
          CMRU & 5000 & \shortstack{0.0470 \\ {\scriptsize [0.0355; 0.0618]}} \\
          $\alpha$CMRU & 5000 & \shortstack{0.0002 \\ {\scriptsize [0.0002; 0.0002]}} \\
          LRU & 5000 & \shortstack{0.5038 \\ {\scriptsize [0.5038; 0.5039]}} \\
          minGRU & 5000 & \shortstack{0.5038 \\ {\scriptsize [0.5038; 0.5038]}} \\
          CMRU & 10000 & \shortstack{0.0589 \\ {\scriptsize [0.0538; 0.0698]}} \\
          $\alpha$CMRU & 10000 & \shortstack{0.0002 \\ {\scriptsize [0.0002; 0.0002]}} \\
          LRU & 10000 & \shortstack{0.5038 \\ {\scriptsize [0.5038; 0.5039]}} \\
          minGRU & 10000 & \shortstack{0.5038 \\ {\scriptsize [0.5038; 0.5038]}} \\
          \bottomrule
        \end{tabular}
      \end{sc}
    \end{small}
  \end{center}
  \vskip -0.1in
\end{table}


\end{document}